\documentclass{article} 

\usepackage{amsmath,amsfonts,bm}

\def\1{\bm{1}}

\def\vm{{\bm{m}}}

\def\vv{{\bm{v}}}

\def\vx{{\bm{x}}}
\def\vy{{\bm{y}}}
\def\vz{{\bm{z}}}

\DeclareMathAlphabet{\mathsfit}{\encodingdefault}{\sfdefault}{m}{sl}
\SetMathAlphabet{\mathsfit}{bold}{\encodingdefault}{\sfdefault}{bx}{n}

\def\der{{\mathrm{d}}}

\newcommand{\E}{\mathbb{E}}
\newcommand{\Ps}{\mathcal{P}}
\newcommand{\Ls}{\mathcal{L}}
\newcommand{\Is}{\mathcal{I}}
\newcommand{\Gs}{\mathcal{G}}
\newcommand{\Os}{\mathcal{O}}

\newcommand{\Ms}{\mathcal{M}}
\newcommand{\R}{\mathbb{R}}

\DeclareMathOperator*{\argmin}{arg\,min}
\DeclareMathOperator*{\iid}{\overset{\mathrm{iid}}{\sim}}

\usepackage[left=1in,top=1in,right=1in,bottom=1in,letterpaper]{geometry}
\usepackage{hyperref}
\usepackage{url}
\usepackage{microtype}
\usepackage{graphicx}
\usepackage{subfigure}
\usepackage{booktabs} 
\usepackage{amsmath}
\usepackage{amsfonts}
\usepackage{amsthm}
\usepackage{xspace}
\usepackage{soul}
\usepackage{mathtools}
\usepackage[square,sort,comma,numbers]{natbib}
\usepackage{breqn}
\usepackage[T1]{fontenc}
\usepackage[font=small,labelfont=bf,tableposition=top]{caption}
\usepackage[bottom]{footmisc}
\usepackage{tablefootnote}
\usepackage[dvipsnames]{xcolor}
\usepackage{float}
\usepackage{mathtools}

\DeclareCaptionLabelFormat{andtable}{#1~#2  \&  \tablename~\thetable}

\newtheorem{thm}{Theorem}[section]

\newtheorem{prop}[thm]{Proposition}
\newtheorem{lem}[thm]{Lemma}

\newtheorem{defn}[thm]{Definition}
\newtheorem{remark}[thm]{Remark}

\newlength\mylength
\newlength\mylengthLong
\title{Unsupervised Solution Operator Learning for Mean-Field Games via Sampling-Invariant Parametrizations\footnote{This work is supported in part by NSF DMS-2401297.}}

\author{Han Huang \thanks{H. Huang (huangh14@rpi.edu) is with the department of mathematics, Rensselaer Polytechnic Institute. }\qquad 
Rongjie Lai \thanks{Corresponding author. R Lai (lairj@purdue.edu) is with the department of mathematics, Purdue University. }
}

\begin{document}

\date{}
\maketitle

\begin{abstract}


Recent advances in deep learning has witnessed many innovative frameworks that solve high dimensional mean-field games (MFG) accurately and efficiently. These methods, however, are restricted to solving single-instance MFG and demands extensive computational time per instance, limiting practicality. To overcome this, we develop a novel framework to learn the MFG solution operator. Our model takes MFG instances as input and output their solutions with one forward pass. To ensure the proposed parametrization is well-suited for operator learning, we introduce and prove the notion of sampling-invariance for our model, establishing its convergence to a continuous operator in the sampling limit. Our method features two key advantages. First, it is discretization-free, making it particularly suitable for learning operators of high-dimensional MFGs. Secondly, it can be trained without the need for access to supervised labels, significantly reducing the overhead associated with creating training datasets in existing operator learning methods. We test our framework on synthetic and realistic datasets with varying complexity and dimensionality to substantiate its robustness. Compared to single-instance neural MFG solvers, our approach reduces the time to solve a MFG problem by more than five orders of magnitude without compromising the quality of computed solutions.

\end{abstract}

\section{Introduction}

Mean-field games (MFG) have emerged as a powerful mathematical framework with diverse applications in game theory~\cite{MFG_game_thry, Learning_MFG}, economics~\cite{MFG_econ1, MFG_econ2}, and industrial planning~\cite{MFG_industry1, MFG_industry2, MFG_industry3}. At its core, MFG is the study of strategic interactions among a large number of rational agents whose actions collectively impact the overall system, offering a bridge between micro-level decision-making and macro-level dynamics. To obtain the strategic equilibria of a MFG, one can derive a forward and backward system composed of a Hamilton-Jacobi-Bellman (HJB) equation~\cite{NN_MFP} and a Fokker-Planck equation evolving bidirectionally in time that prescribes the MFG's optimiality condition. A plethora of numerical methods can then be employed to solve this system~\cite{achdou2010mean, benamou2014augmented,benamou2017variational,benamou2000computational, jacobs2019solving,papadakis2014optimal,yu2021fast}. While these methods are accurate and sometimes provably convergent, they scale poorly with respect to problem dimension due to their reliance on spatial discretization. 

Recent strides in solving MFG have seen a paradigm shift with the advent of deep learning techniques~\cite{NN_MFP, MFG_NF, Tong_APACNet}. A common underpinning for these methods is to parametrize key terms in the MFG system, e.g., the value function or the agent trajectory, with suitable neural architectures and use black-box optimization to search for the solution within a class of functions. If the architecture is expressive enough, the surrogate serves as a good approximation for the MFG optima. Compared to traditional methods, deep learning approaches sidesteps spatial discretization with parametrization and can accurately solve MFG in hundreds or even thousands of dimensions~\cite{schodinger_bridge_MFG}.

While deep learning has been successful in tackling previously intractable MFG problems, current approaches are calibrated to solve \textit{single instances} of MFG. As a motivating example, consider a motion planning problem that seeks to efficiently move a group of drones to their destination while avoiding spatial obstacles. Current approaches train a network to solve this MFG problem with fixed initial and terminal drone locations. However, were the drones to start or end at a different location, the previous optimization is rendered obsolete, and a new network must be trained from scratch. As the training takes the order of hours to days~\cite{NN_MFP, MFG_NF}, the existing methods are severely limited in their applicability. 


To bridge this gap, we propose a novel framework that uses a single network to solve different MFG problems without the need for retraining. Our approach trains on a collection of MFG configurations and learns the \textit{solution operator} that maps between the problem setup and solution. The main challenge here lies within parametrizing a mapping between infinite dimensional spaces. Recently, neural operators~\cite{FNO, neural_op} has emerged as a prominent framework for learning solvers for partial differential equations (PDEs). Naively, one can choose a discretization grid and represent the input and output with their density values on a finite number of points. This approach, however, scales poorly with dimensionality and sacrifices the very advantage that deep learning approaches are crowned upon. In addition, the training of neural operators relies on ground truth solutions, which are particularly expensive to obtain for MFG problems.


To overcome these challenges, we employ finite samples to represent agent distributions. Our architecture integrates attention mechanism. It processes input samples of any size consistently, offering a discretization-free and scalable approach for high-dimensional environments. We also introduce a concept,  \textit{sampling invariance}, which characterizes advantages of our network parametrization. Additionally, we introduce a novel objective that minimizes the collective MFG energy across problem instances. Compared to alternative supervised learning approaches, our method allows for the direct parameterization of the solution operator and can be optimized without supervised labels, broadening its applicability across numerous scenarios. Our framework's effectiveness is demonstrated through comprehensive numerical experiments on both synthetic and real-world datasets. Our model successfully learns precise and intuitive MFG solutions across various levels of complexity and dimensionality, marking the first computational method for learning high-dimensional MFG solution operators in the unsupervised manner. By training on the distribution of MFG instances, our method can solve new problems at inference time without further weight updates, which transforms the time to solve a MFG from hours to essentially real-time.  

\paragraph{Contributions}

\begin{itemize}
    \item We develop a framework for learning high dimensional MFG solution operators. 
    \item Our approach features a novel objective that can be optimized without ground truth labels. We prove its minimizers are MFG solution operators and remark that it may be adapted to enable unsupervised operator learning for other computational problems.
    \item We introduce and prove the notion of sampling invariance for our architecture. The definition is general and applicable to settings where sampling based representations are natural.
    \item We derive analytical solutions for a special MFG problem to illustrate our method's scalability with regard to dimension. It also serves as a useful test case for neural MFG solvers.
    \item We compare our approach to single-instance MFG solvers to underscore its comparable accuracy and immense reduction in inference time - about \textbf{five orders of magnitude} faster for solving a new MFG problem.
\end{itemize}

\subsection{Related Works}
\label{sec:relaedworks}

\paragraph{MFG and OT}Variational mean field game (MFG) generalizes the dynamic formulation of optimal transport (OT)~\cite{OT_book}, where the final density matching is not strictly enforced and additional transport costs beyond kinetic energy can be considered. While traditional methods for solving OT and variational MFG problems are well-established in lower dimensions ~\cite{achdou2010mean, benamou2014augmented,benamou2017variational,benamou2000computational,jacobs2019solving,papadakis2014optimal,yu2021fast,cuturi2013sinkhorn, yu2023computational}, such methods typically require spatial discretization, so their memory consumption scales exponentially with respect to dimensionality.

Recent advances in machine learning have enabled the solving of high dimensional MFG and OT. For MFG, \cite{NN_MFP} approximates the value function of deterministic MFGs using deep neural networks, while \cite{MFG_NF} discretizes and parametrizes agent trajectories with normalizing flows. In addition, \citep{Tong_APACNet} solves stochastic MFGs with adversarial training, and \cite{schodinger_bridge_MFG} leverages theory in Schrodinger Bridge and techniques in reinforcement learning to solve MFGs with potentially discontinuous interaction costs. For OT, \cite{ICNN_OT} leverages input convex neural networks (ICNN) to parametrize dual potentials and solves the OT problem via minimax optimization. \cite{W2GAN} builds on the input convex framework and uses a cycle consistency loss to sidestep adversarial training. 

While the aforementioned neural MFG solvers can produce accurate solutions for individual MFG problems, their optimization process must be reinitiated for every MFG instance, rendering them impractical for solving large collections of MFGs.

\paragraph{Operator Learning}
A prominent framework for parametrizing and learning maps between function spaces is neural operators \cite{neural_op}, whose main application concerns the solution of partial differential equations (PDE). Prominent architectures in the neural operator family include the Fourier neural operator and the Multipole graph neural operator \cite{FNO, MGNO}, which are recognized for their efficiency and invariance to domain discretization.

To our knowledge, the sole existing method for learning MFG solution operators is detailed in~\cite{MFG_master_policy}. It introduces the master policy - a strategy that allows a representative agent to play optimally against any population distribution, then leverages reinforcement learning techniques to approximate such policy. However, this approach is tailored to MFG with discrete action and state spaces, so its adaptation to continuous settings necessitates expensive state space discretization. In contrast, our work directly addresses continuous dynamics and is more scalable to higher spatial dimensions.

Similarly, we know of only one approach for computing the solution map for OT. The work~\cite{meta_OT} leverages a hypernetwork~\cite{hypernetwork2} to predict weights that parametrize optimal couplings for a set of OT problems. Crucially, this approach hinges on utilizing input-convex neural networks \cite{ICNN_OT} to parametrize the dual potential of OT, which neglects the problem's dynamical aspect and thus cannot generalize to the MFG setting.

Recently, ~\cite{ICON} proposes in-context operator network(ICON), a single transformer used to parametrize an "operator of operators" as a general, problem-agnostic framework to solve different PDE problems. It leverages in-context learning for the trained model to identify the operator in question from few-shot demonstrations consisting of observational data, then make predictions for the queried quantities of interest. ICON has been tested on a few MFG problems in one and two dimensions, showing promising performance. We highlight two crucial differences between our framework and ICON. Firstly, training ICON requires ground truth solutions for the desired problem, which can be costly to obtain. In contrast, our framework minimizes the amortized MFG cost and can be trained in an unsupervised fashion. Second, ICON has only been tested on low-dimensional MFGs, so its performance in higher dimensions is unknown. 

The rest of this paper is organized as follows. In Section \ref{sec:background}, we delve into the background of MFGs and highlight the limitations associated with directly applying existing operator learning methods to address MFGs.
Section \ref{sec:methodology} introduces a novel framework for unsupervised learning of solvers for MFGs. This includes a detailed exploration of the strategy for sampling representation and network architecture.
Moving on to Section \ref{sect:theo_results}, we state theoretical analyses of our proposed method. We demonstrate the sample invariance of our network and establish the theoretical soundness of the unsupervised framework.
In Section \ref{sec:results}, we present comprehensive experimental results to assess the effectiveness of our proposed method in solving MFGs. In Section \ref{sec:proofs}, we provide proof details for all theoretical statements. 
Finally, Section \ref{sec:conclusion} concludes the paper, summarizing key findings and potential avenues for future research.

\section{Background}
\label{sec:background}
\subsection{Mean-Field Games}

The study of MFGs considers symmetric $N$-player games at the limit of $N\to \infty$~\cite{MFG_varMFG,huang2006large,huang2007large}. Assuming homogeneity of the objectives and anonymity of the players, we set up an MFG with a continuum of non-cooperative rational agents distributed spatially in $\mathbb{R}^d$ and temporally in $[0, T]$. For any fixed $t\in [0,T]$, we denote the population density of agents by $p(\cdot, t)$. For an agent starting at $\vx_0\in \R^d$, their position over time follows a trajectory $\vx: [0,T] \to \R^d$ governed by
\begin{equation} \label{agent_traj}
\left\{\begin{aligned}
    \der \vx(t) &= \vv(\vx(t),t)\der t, \quad \forall t\in [0,T]\\
    \vx(0) &= \vx_0,
\end{aligned}\right.
\end{equation}
where $\vv: \R^d \times [0,T] \to \R^d$ specifies an agent's action at a given time. For simplicity, we assume no stochastic terms in~\eqref{agent_traj}. As a result, each agent's trajectory is completely determined by $\vv(\vx,t)$. To play the game over an interval $[t, T]$, each agent seeks to minimize their objective:
\begin{equation} \label{MFG}
\begin{aligned}
    J_{\vx_0,t }(\vv,p) &\coloneqq \int_t^\top  [L(\vx(s),\vv(\vx(s),s)) + I(\vx(s), p(\vx(s), s)) ] \der s + M(\vx(T), p(\vx(T), T))\\
    &\text{s.t. } \eqref{agent_traj} \text{ holds}.
\end{aligned}
\end{equation}
The transport cost $L: \R^d \times \R^d \to \mathbb{R}$ is incurred by each agent's own action. A common example is the kinetic energy $L(\vx,\vv) = \|\vv\|^2$, which accounts for the total amount of movement along the trajectory. The interaction cost $I: \R^d \times \mathcal{P}(\R^d) \to \mathbb{R}$ is accumulated through the agent interacting with another agent or with the environment. For instance, one can consider an entropy term that discourages grouping behavior, or a penalty for colliding with an obstacle~\cite{NN_MFP}. The terminal cost $M: \R^d \times \mathcal{P}(\R^d) \to \mathbb{R}$ is computed from the agents' final state, which typically measures a discrepancy between the final measure $P(\cdot, T)$ and the desired $P_1 \in \mathcal{P}(\R^d)$.


Under suitable assumptions, the seminal work~\cite{MFG_varMFG} established an equivalence between MFG and mean-field control problems. Suppose there exist functionals $\mathcal{I}, \mathcal{M}: \mathcal{P}(\R^d) \to \mathbb{R}$ such that
\begin{align*}
    I(\vx,p) = \frac{\delta \mathcal{I}}{\delta p}(\vx), \quad M(\vx,p) = \frac{\delta \mathcal{M}}{\delta p}(\vx),
\end{align*}
where $\frac{\delta}{\delta p}$ is the variational derivative. Then, the functions $p(\vx,t)$ and $\vv(\vx,t)$ satisfying~\eqref{MFG} coincide with the optimizers of the following variational problem: 
\begin{equation}\label{var_MFG}
\begin{split}
    \inf_{p,\vv} J(p,\vv) \coloneqq \int_0^\top  \int_{\R^d} L(\vx, \vv(\vx,t),p(\vx,t) \der\vx \der t &+ \int_0^\top  \mathcal{I} (p(\cdot, t))\der t + \mathcal{M}(p(\cdot, T))\\
    \text{s.t. } \quad \partial_t p(\vx,t) + \nabla_{\vx} \cdot (p(\vx,t)\vv(\vx,t)) &= 0, \quad \vx\in \R^d, t\in [0,T]\\
    p(\vx,0) = p_0(\vx), \quad \vx\in &\R^d.
\end{split}
\end{equation}
This formulation is termed the \emph{variational MFG}. Unless specified otherwise, we assume $T=1$ and use the $L_2$ transport cost $ L(\vx,\vv(\vx,t),p(\vx,t)) = \lambda_L p(\vx,t) \|\vv(\vx,t)\|_2^2$, where $\lambda_L \geq 0$ is a hyperparameter.

\subsection{Trajectory-Based MFG}
A key challenge in solving the variational problem~\eqref{var_MFG} lies in enforcing the continuity equation. To sidestep this, a trajectory-based reformulation of the variational MFG~\eqref{var_MFG} is derived in \cite{MFG_NF} to remove the PDE constraint. Let $P(\cdot,t)$ be the measure that admits $p(\cdot,t)$ as its density for all $t\in [0,T]$. Define the agent trajectory as $F: \mathbb{R}^{d} \times \mathbb{R} \to \mathbb{R}^d$, where $F(\vx,t)$ is the position of the agent starting at $\vx$ and having traveled for time $t$. It satisfies the following differential equation:
\begin{equation} \label{NF_traj}
\left\{\begin{aligned}
    \partial_t F(\vx,t) &=  \vv(F(\vx,t), t), \quad \vx\in \R^d, t\in[0,T] \\
    F(\vx,0) &= \vx, \quad \vx\in \R^d.
\end{aligned}\right.
\end{equation}
The evolution of the population density is determined by the movement of agents. Thus, $P(\cdot,t)$ is simply the push-forward of $P_0$ under $F(\cdot, t)$; namely, $P(\vx, t) = F(\cdot, t)_* P_0(\vx)$, whose associated density satisfies
\begin{equation} \label{density_push-forward}
\begin{aligned}
    p(\vx,t) &= \mathrm{d}(F(\cdot, t)_* P_0)(\vx),
\end{aligned}
\end{equation}
where $\mathrm{d}(F(\cdot, t)_* P_0)(\vx) \coloneqq F(\cdot, t)_* p_0(\vx)$ is a Radon-Nikodym derivative. As a result, we can apply the change of measure to transform \eqref{var_MFG} into
\begin{equation} \label{var_MFG_traj}
\begin{aligned}
    \inf_{F} \quad \lambda_L \int_0^1 \int_{\R^d}  \|\partial_t F(\vx,t)\|_2^2 p_0(\vx)\der\vx \der t &+ \lambda_{\mathcal{I}} \int_0^1 \mathcal{I}(F(\cdot, t)_*P_0) \der t + \lambda_{\mathcal{M}} \mathcal{M}(F(\cdot, T)_*P_0) \coloneqq \mathcal{L} (P_0,P_1,F)\\
    &\text{s.t. } \quad  F(\vx,0) = \vx.
\end{aligned}
\end{equation}

As alluded to previously, trajectory MFG~\eqref{var_MFG_traj} absolves the continuity-equation constraint into the pushforward definition of $p(\vx,t)$. The transformed problem has a much simpler constraint that can be automatically satisfied with suitable parametrizations of $F$. For example , normalizing flows are chosen to represent the trajectory map in~\cite{MFG_NF}.

\subsection{Interaction-free MFG and Optimal Transport}

It is known that in the special case where $\Is \equiv 0$, the MFG~\eqref{var_MFG_traj} is equivalent to a related optimal transport (OT) problem~\cite{MFG_NF}:
\begin{equation} \label{monge_MFG}
\begin{aligned}
    \inf_{T} \quad \lambda_L \int_{\R^d}  \| T(\vx) - \vx \|_2^2 p_0(\vx)\der\vx  &+ \lambda_{\mathcal{M}} \mathcal{M}(T_*P_0)
\end{aligned}
\end{equation}

One may derive the above problem by observing the minimizer of interaction-free MFG has the representation $F^*(\vx,t) = (1-t)\vx + tT^*(\vx)$, i.e., optimal trajectories are straight, then substitute and recast the optimization in terms of $T$. Conceptually, $T$ maps each agent directly to their respective destination and is referred to as the Monge map in the OT literature~\cite{OT_book}. Crucially, this formulation removes time dependent dynamics as well as the initial value constraint from the target mapping, making the problem much simpler to solve. In addition, ~\ref{monge_MFG} is equivalent to the OT problem if we take $\lambda_\Ms \to \infty$.


\subsection{MFG Solution Operator}

We first define the MFG solution operator, the central object of interest of this work.
\begin{defn}
    The MFG solution operator, $\Gs^*$, is defined as $\Gs^*:(P_0, P_1) \mapsto F^* \in \argmin_F \mathcal{L}(P_0,P_1,F)$.
\end{defn}
Intuitively, $\Gs^*$ takes a pair of initial and terminal agent distributions and outputs optimal MFG trajectories that minimizes the costs accumulated during their course of travel. There are no existing methods for learning $\Gs^*$, through operator learning has been considered in other settings. We present two relevant approaches here inspired by some existing works and remark on their respective drawbacks.


A generic approach for obtaining operators such as $\Gs^*$ is found in the neural operator literature \cite{neural_op}. We first obtain a collection of data $\{(P_0^n, P_1^n)\}_{n=1}^N$ and corresponding labels $\{F^n\}_{n=1}^N$, where $((P_0^n, P_1^n), F^n) \overset{\mathrm{iid}}{\sim} \eta$, then seek a solution that minimizes its discrepancy on the training data over space and time: 

\begin{equation} \label{operator_training_w_data}
\begin{aligned}
    \min_{\Gs} \quad \E_{((P_0, P_1), F) \sim \eta} \int_0^1 \|\Gs(P_0, P_1)(\cdot, t) - F(\cdot, t)\|^2_{L_2(\R^d)} \der t
\end{aligned}
\end{equation}


We term~\eqref{operator_training_w_data} the \textit{interpolation-based} objective since it curve-fits the desired operator based on training data. A critical drawback with this approach, however, is its reliance on supervised labels. For reference, typical MFG solvers take minutes to hours to obtain accurate solutions, so obtaining training labels for a large dataset is prohibitively expensive~\cite{yu2021fast}. The situation becomes more dire in high dimensional ($d>3$) settings as classic numerical solvers are simply unusable.

An alternative approach for solution map learning is inspired by a method discussed in~\cite{meta_OT}. Let $F_\theta(\vx, t)$ denote the parametrized agent flow, and let $G_\phi: (P_0, P_1) \mapsto \theta$ be a \textit{hypernetwork} \cite{hypernetwork, hypernetwork2} whose output are weights for the trajectory parameterization network $F_\theta$. Suppose $\{(P_0^n, P_1^n)\}_{n=1}^N$ is a collection of initial and terminal measures sampled from $\mu$, meta-OT considers the following problem (adapted for MFG):

\begin{equation} \label{meta_OT_obj_MFG}
\begin{split}
    \min_{\phi} \quad \E_{(P_0, P_1) \sim \mu} \lambda_L \int_0^1 \int_{\R^d}  \|\partial_t F_\theta(\vx,t)\|_2^2 &p_0(\vx)\der\vx \der t + 
    \lambda_{\mathcal{I}} \int_0^1 \mathcal{I}(F_\theta(\cdot,t)_*P_0) \der t + \lambda_{\mathcal{M}} \mathcal{M}(F_\theta(\cdot,T)_*P_0)\\
    \text{s.t. } \quad  &F_\theta(\vx,0) = \vx.\\
    &G_\phi(P_0, P_1) = \theta
\end{split}
\end{equation}

We call~\eqref{meta_OT_obj_MFG} the \textit{prediction-based}~\cite{pc_OT_reg} objective since it leverages an intermediate model to predict the parameters used for solving each problem instance. While ~\eqref{meta_OT_obj_MFG} can be optimized without training labels, it is difficult to scale because $F_\theta$ may require millions of parameters to solve high dimensional MFG problems \cite{MFG_NF}, rendering $G_\phi$ infeasibly large. In addition, it is known that hypernetworks struggle with precise predictions and have diminishing returns in model and data sizes~\cite{pc_OT_reg}. 

Besides the training objective, we also need an appropriate parametrization for $\Gs^*(P_0,P_1)$. Contrary to typical deep learning settings, $P_0,P_1$ are infinite dimensional objects, so the architecture for $\Gs^*$ warrants special consideration. In the conventional neural operator approach, one may consider the associated density functions to parametrize $\Gs^*$ as a composition of kernel integral operators~\cite{FNO,neural_op}. However, this approach is cursed by dimensionality since the kernel integrals need to be discretized, and the number of points needed to properly cover the spatial domain scales exponentially.


In summary, neural operator learning for MFG solution maps is an unexplored area of research. While adjacent approaches exist, their direct adaptation either incurs expensive computational overhead or suffers from accuracy and scalability limitations. Acknowledging these obstacles, we shall suggest a different approach to sidestep the challenges.

\section{Methodology}
\label{sec:methodology}


The ultimate goal of this work is to parametrize $\Gs^*(P_0, P_1)$ with a neural network and obtain the MFG solution associated with $(P_0,P_1)$ via a single forward pass through the model. To achieve this, we first develop a novel optimization framework that is scalable and allows for the unsupervised learning of $\Gs^*$. Next, we outline a sampling-based finite-dimensional representation of $P_0, P_1$ and propose suitable network architectures based on this representation.  


\subsection{An Unsupervised Operator Learning Framework}

We unveil a novel learning framework that combines merits of interpolation-based~\eqref{operator_training_w_data} and prediction-based~\eqref{meta_OT_obj_MFG} training to enable unsupervised and accurate operator learning. The key observation is that the prediction-based objective~\eqref{meta_OT_obj_MFG} introduces parametrized mappings, i.e. $F_\theta:\mathbb{R}^d\rightarrow\mathbb{R}^d$, into the formulation too prematurely. It pays to adhere to the operator formalism and delay consideration for parametrization until later. In this spirit, we directly consider paramterize solution mapping $\mathcal{G}$ by consider the following optimization problem:

\begin{equation} \label{MFG_oper_obj}
\begin{split}
    \min_{\Gs} \quad \E_{(P_0, P_1) \sim \mu} \lambda_L \int_0^1 \int_{\R^d}  \|\partial_t \Gs(P_0, P_1)(\vx,t)\|_2^2 &p_0(\vx)\der\vx \der t + 
    \lambda_{\mathcal{I}} \int_0^1 \mathcal{I}(\Gs(P_0, P_1)(\cdot,t)_*P_0) \der t +  \\
    &\hspace{10mm} \lambda_{\mathcal{M}} \mathcal{M}(\Gs(P_0, P_1)(\cdot,T)_*P_0) \coloneqq \Ls_s(P_0,P_1,\Gs(P_0,P_1))\\
    \text{s.t. } \quad  &\Gs(P_0, P_1)(\vx,0) = \vx.
\end{split}
\end{equation}


This is the central optimization problem this work seeks to solve. Similar to~\cite{meta_OT}, we call it the \textit{amortized} objective since it minimizes the MFG cost over a distribution of configurations. Compared to alternatives, \eqref{MFG_oper_obj} can be optimized without knowledge of the true MFG solutions and allows direct parametrization of operator $\Gs$ to sidestep scalability and accuracy bottlenecks in hypernetworks.

In the absence of interaction cost, i.e., $\Is \equiv 0$, it is known that the MFG \eqref{var_MFG_traj} is equivalent to a related OT problem \cite{MFG_NF}. Thus, the optimal agent trajectories are straight and can be rewritten in terms of the Monge map $T^*(\vx)$, i.e.,  $F^*(\vx,t) = (1-t)\vx + t T^*(\vx)$. In this case, we instead study the time-independent solution operator $\Gs(P_0,P_1) = T^*$, simplifying the operator learning problem~\eqref{MFG_oper_obj} to:

\begin{equation} \label{OT_oper_obj}
\begin{split}
    \min_{\Gs} \quad \E_{(P_0, P_1) \sim \mu} \lambda_L \int_{\R^d}  \|\Gs(P_0, P_1)(\vx) - \vx\|_2^2 &p_0(\vx)\der\vx + \lambda_{\mathcal{M}} \mathcal{M}(\Gs(P_0, P_1)_*P_0)
\end{split}
\end{equation}

Notably, the Monge map representation automatically satisfies $\Gs(P_0, P_1)(\vx,0) = \vx$, so the optimization problem is unconstrained.

\subsection{Input Representation}



In general, $P_0, P_1$ reside in an infinite dimensional probability measure space $\mathcal{P}(\mathbb{R}^d)$. As such, we have to use their finite dimensional representations as inputs to a parametrized model. In this work, we take advantage of $P_0, P_1$ being probability measures to represent them with their iid \textit{samples}. Indeed, assuming access to agent samples is arguably more natural than agent densities in realistic settings. For example, in drone motion planning~\cite{Tong_APACNet}, it is more straightforward to obtain individual drone positions rather than the concentration of drones at different locations.

\begin{figure}[h]
\centering
\begin{minipage}{0.25\linewidth}
\centering
\includegraphics[width=1\linewidth]{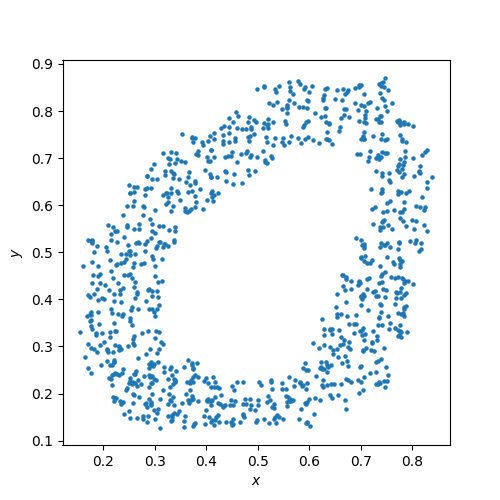}
\end{minipage}\hfill
\begin{minipage}{0.25\linewidth}
\centering
\includegraphics[width=1\linewidth]{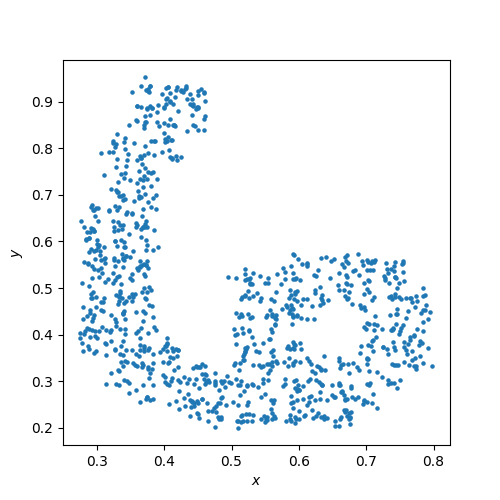}
\end{minipage}\hfill
\begin{minipage}{0.25\linewidth}
\centering
\includegraphics[width=1\linewidth]{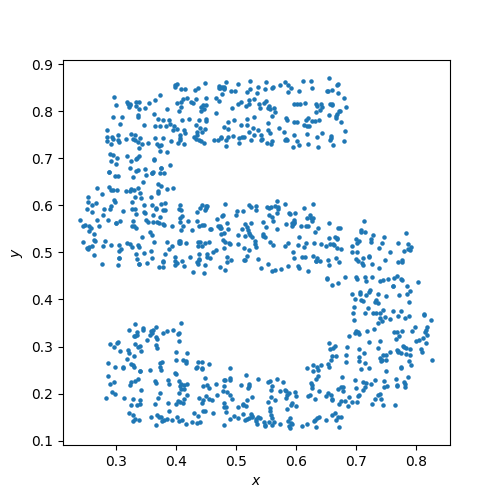}
\end{minipage}\hfill
\begin{minipage}{0.25\linewidth}
\centering
\includegraphics[width=1\linewidth]{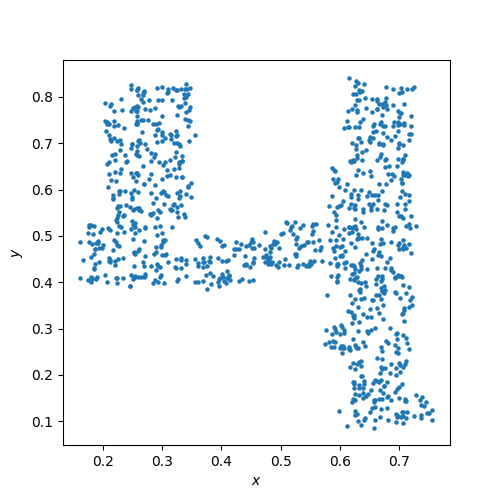}
\end{minipage}\hfill
\vskip -10pt
\caption{MNIST digits 0,6,5,4 represented by 1053 samples from their pixel value densities.}
\label{fig:MNIST_digits}
\end{figure}

Concretely, we assume access to $\{\vx_0^i\}_{i=1}^n \iid P_0, \{\vx_1^i\}_{i=1}^n \iid P_1$. Denote stacked samples to be $[X_0]_{i,:} = \vx_0^i, [X_1]_{i,:} = \vx_1^i$. One may think of $X_0, X_1$ as two point clouds that reflect the general shape of $P_0, P_1$ - several examples are provided in Figure~\ref{fig:MNIST_digits} in which the distributions are MNIST digits. Our sampling-based representation is well-suited for high dimensional settings. Intuitively, $X_0, X_1$ are concentrated on the high density regions of $P_0, P_1$, which are precisely the main contributors for the cost terms. This is akin to using an adaptive grid that becomes more refined in areas with higher agent density. 

Furthermore, we make a few remarks on the terminal cost $\Ms$. Typically, $\Ms$ serves as an incentive for agents to stay close to the terminal distribution $P_1$. From this perspective, any $\Ms(P)$ is deemed admissible so long as it is non-negative and zero if and only if $P=P_1$. However, one must also consider the computational tractability of $\Ms$. In our setting, we assume only access to samples from $P_0, P_1$ and no information on their density functions. Therefore, a suitable $\Ms$ to use is the maximum mean discrepancy (MMD), which admits the following unbiased estimator~\cite{MMD_net}

\begin{equation}\label{MMD_estimator}
\begin{split}
    \mathrm{MMD}^2(P, Q) = \frac{1}{N(N-1)} \sum_{n\ne n'} k(\vx^n, \vx^{n'}) - \frac{2}{MN} \sum_{m=1}^M \sum_{n=1}^N k(\vx^n, \vy^m) + \frac{1}{M(M-1)} \sum_{m\ne m'} k(\vy^m, \vy^{m'})
\end{split}
\end{equation}
where $\{\vx^n\}_{n=1}^N \iid P, \{\vy^m\}_{m=1}^M \iid Q$, and $k(\vx,\vy)$ is a symmetric kernel, i.e., $k(\vx,\vy) = k(\vy,\vx), \forall \vx,\vy\in \R^d$, and $\sum_{i,j=1}^d \alpha_i \alpha_j k(\vx_i,\vx_j) \geq 0, \forall (\alpha_1, ..., \alpha_d)\in \R^d, \vx_i\in \R^d$. By the Moore-Aronszajn theorem \cite{moore_aronszajn}, every kernel defines a reproducing kernel Hilbert space (RKHS), and if its mean embedding is injective, the associated MMD defines a metric on the space of probability measures, making it a suitable candidate for $\Ms$~\cite{characteristic_kernel}. Such kernels are called \textit{characteristic}, and common examples include the Gaussian kernel: $k(\vx,\vy) = \exp{\frac{-\|\vx-\vy\|^2_2}{2\sigma^2}}$ and the Laplacian kernel: $k(\vx,\vy) = \exp{-\lambda\|\vx-\vy\|_1}$, where $\sigma, \lambda \in \R$ are fixed bandwidth hyperparameters. 

There are two main advantages for using the MMD as the terminal cost. First, its estimator can be computed in closed form without any architectural constraints. In contrast, estimating the KL divergence, another popular choice for $\Ms$, requires invertible parametrizations for applying the change of variable~\cite{NF_survey}. Second, ~\cite{MMD_kernel_test} has shown that the estimator~\eqref{MMD_estimator} converges to the true MMD in probability at $\mathcal{O}((M+N)^{-\frac{1}{2}})$, a rate independent of dimensionality. For a more in-depth exposition on MMD as an integral probability metric as well as its applications in statistics and machine learning, see \cite{MMD_kernel_test, MMD_survey}.

\subsection{Operator Parametrization}\label{sec:operator_parametrization}

In this section, we examine appropriate parametrizations for solving the operator learning problem. With finite-dimensionalized inputs $(P_0, P_1) \to (X_0, X_1)$, the proposed model has the general form $G_\theta(X_0, X_1)(\vx,t)$, where $\theta$ denotes trainable weights. Our architectural design is motivated by two guiding principles: permutation and sampling invariance. 

First, since $X_0, X_1$ are stacked iid samples from $P_0, P_1$, their row orderings are arbitrary. Thus, outputs of the solution operator must be (row) \textit{permutation invariant}, and we ought to inject this prior knowledge into our parametrization so that $G_\theta$ exhibits this property for all choices of $\theta$.
Formally, let's make the sample size $n$ explicit in our notation of stacked samples $X^n_0, X^n_1 \in \R^{n\times d}$, we define permutation invariance as follows:

\begin{defn}[Permutation Invariance]
    Let $X^n_0, X^n_1\in \R^{n\times d}$. The parametrization $G_\theta(X^n_0, X^n_1)(\vx,t)$ is permutation invariant if $G_\theta(\tilde{X^n_0}, \tilde{X^n_1})(\vx,t) = G_\theta(X^n_0, X^n_1)(\vx,t)$ for any $\tilde{X^n_i} = \sigma(X^n_i)$, i=0,1, where $\sigma$ is a permutation on the rows.
\end{defn}

Second, any sensible parametrization for an operator should observe a few basic properties. For example, the model should process different finite-dimensionalizations of the same input consistently. Specifically, we ask that the model can accept finite representations of any resolution as input and, as the resolution increases, converges to a continuous operator. Additionally, the model's output is a continuous mapping, so we should be able to query the output mapping on any point on its domain. We summarize these properties under the umbrella term \textit{sampling invariance} defined as follows:

\begin{defn}[Sampling invariance]\label{defn:sampling_inv}
Consider a MFG with initial and terminal distributions $P_0, P_1$. Let $X^n_0, X^n_1\in \R^{n\times d}$ be row-stacked iid samples of $P_0,P_1$. The parametrization $G_\theta(X^n_0, X^n_1)(\vx, t)$ is sampling invariant if, for a fixed $\theta$, 
    \begin{itemize}
        \item $G_\theta(X^n_0, X^n_1)(\vx, t)$ is defined $\forall n\in \mathbb{N}, (\vx, t) \in \R^d \times [0,T]$.
        \item $\lim_{n\to\infty} G_\theta(X^n_0, X^n_1)(\vx, t) = \Gs(P_0,P_1)(\vx,t)$ for some operator $\Gs$.
    \end{itemize}
\end{defn}

From here on, we remove the superscript $n$ on $X^n$ for conciseness. A simple architecture that satisfies both sampling and permutation invariance is the point-wise multi-layer perceptron (MLP)~\cite{PINN, PointNet}. Concretely, let $f_\phi: \R^d \to \R^d$ be a MLP:
\begin{align*}
    f_\phi(\vx) = W_k \sigma_{k-1}(\ldots \sigma_2(W_2\sigma_1(W_1 \vx + b_1) + b_2) \ldots) + b_k,
\end{align*}
where $\phi = (W_1, ..., W_k, b_1, ..., b_k)$ are trainable parameters, and $\sigma_i$ are activation functions. A point-wise MLP $F_\phi:\R^{n\times d} \to \R^{n\times d} $ is obtained by applying $f_\phi$ independently to each row of the input $X\in \R^{n\times d}$:
\begin{align*}
    F_\phi(X)_{i,:} = (f_\phi(X^\top_{i,:}))^\top, \forall i=1,...,n,
\end{align*}
where $X_{i,:}$ denotes the $i$-th row of $X$.  Despite their popularity, we cannot  rely solely on point-wise MLPs because $\Gs^*(P_0, P_1)$ is necessarily non-local. Indeed, the optimal trajectory for an agent must be obtained in relation to the behavior of other agents since MFG solutions are strategic equilibria. Therefore, our architecture must effectively propagate local information to all samples in $X_0, X_1$. To achieve this, we experimented with two popular architectures, spatial convolutions and attention-based blocks \cite{PointNet++, Attention}. Empirically, attention provides superior performance and is thus our model of choice. Formally, a multi-headed attention layer is defined as $\mathrm{MHA}: \R^{n\times d} \to \R^{n\times d}$, where
\begin{align*}
    \text{MHA}(Q, K, V) = \text{Concat}(\text{head}_1, \ldots, \text{head}_h)W^O, \\
    \text{head}_i = \text{Attention}(QW_i^Q, KW_i^K, VW_i^V),\\
    \text{Attention}(Q, K, V) = \text{softmax}\left(\frac{QK^T}{\sqrt{d_k}}\right)V
\end{align*}
Here, $ W_i^Q, W_i^K, W_i^V, W^O$ are trainable matrices and $d_k$ is the number of columns in $K$~\cite{Attention}. Finally, we define a multi-headed attention block $\mathrm{MHT}: \R^{n\times d} \to \R^{n\times d}$ as the composition of a multi-headed attention layer and a point-wise MLP with layer normalization~\cite{layernorm} after both modules.

For ease of presentation, we first give our parametrization for the dynamic OT problem~\eqref{OT_oper_obj}.  In this case, the solution operator $\Gs^*(P_0, P_1)(\vx)$ does not have time in its input. To evaluate $G_\theta(X_0,X_1)(\vx)$, we first concatenate $\vx^\top $ as an extra row of $X_0$, then featurize $\hat{X_0} = \begin{pmatrix}\vx^\top  \\X_0 \\ \end{pmatrix}$ and $X_1$ with separate point-wise MLPs. The featurized points are then concatenated row-wise before entering repeated multi-headed self-attention blocks. Lastly, we take the first row of the output matrix (as it corresponds to $\vx$) to obtain the final output. A schematic for the proposed parametrization is provided in Figure~\ref{fig:architecture}. We note that both MLP and multi-headed self-attention blocks operate on input points individually. Consequently, the network structure is clearly defined for inputs with varying sample points, namely,  matrices with different numbers of rows. We provide details of our network structures in the appendix. 
\begin{figure}[h]
\centering
\begin{minipage}{1\linewidth}
\centering
\includegraphics[width=.9\linewidth]{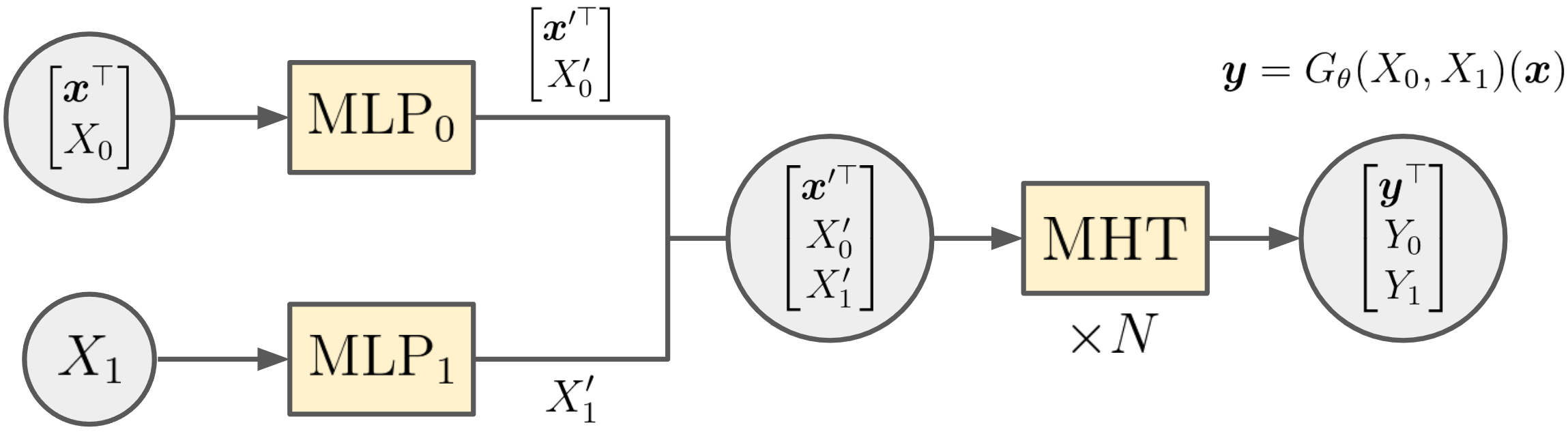}
\end{minipage}\hfill
\caption{Illustration of the proposed parametrization $G_\theta(X_0, X_1)(\vx)$ for MFG solution operators. MLP: Point-wise multi-layer perception. MHT: Multi-headed attention block.}
\label{fig:architecture}
\end{figure}

Note that $\vx$ and $X_0$ are featurized with the same point-wise MLP. This is justified assuming $\vx\sim P_0$, which is reasonable since MFG solution mappings only concern transportation on samples from the initial measure $P_0$. In practice, we generate a set of samples $X_0$ to serve as both the finite-dimensionalized $P_0$ and points on which the training objective is computed. 

For MFG with interaction terms, the operator has a dynamical aspect and takes time $t$ as an additional input. Hence, we make two adaptations to obtain the appropriate model $G_\theta(X_0,X_1)(\vx,t)$. First, we augment the architecture in Figure~\ref{fig:architecture} by concatenating $t$ to $\vx$ and each row of $X_0, X_1$. Second, we need to ensure that the initial value constraint $\Gs_\theta(P_0,P_1)(\vx,0) = \vx$ in~\eqref{MFG_oper_obj} holds. To this end, we put 
\begin{align}
    G_\theta(X_0,X_1)(\vx,t) = \hat{G}_\theta(X_0,X_1)(\vx,t) - \hat{G}_\theta(X_0,X_1)(\vx,0) + \vx
\end{align}
where $\hat{G}_\theta(X_0,X_1)(\vx,t)$ is the parametrization with time augmentation. It is evident that $G_\theta(X_0,X_1)(\vx,0) = \vx$.

\section{Theoretical Results}\label{sect:theo_results}

We supply a several formal statements to substantiate various aspects of our framework. All proofs are provided in Section~\ref{sec:proofs}. 

First, we show that our model is  permutation and sampling invariant.
\begin{thm}\label{thm:sampling_inv}
Let $X^n_0, X^n_1 \in \R^{n \times d}$ with rows sampled iid from $P_0, P_1$, respectively. The proposed parametrization $G_\theta (X^n_0, X^n_1)(\vx,t)$ is permutation and sampling invariant.
\end{thm}

\begin{remark}
There are many reasons to prefer a sampling invariant model. For instance, such models can be trained on inputs of one sample size and used for inference on inputs of any sample size. It is also possible to train sampling invariant models with data of varying sizes, often through efficient implementation with suitable masks.
\end{remark}

Our model's sampling invariance is later verified by studying its empirical behavior on samples of increasing sizes. Next, we formally show the minimizers of our training objective~\eqref{MFG_oper_obj} are MFG solution operators.

\begin{thm}\label{thm:opt_obj}
Denote $\Ps(\Omega)$ to be the space of probability measures over the set $\Omega$. Let $\mu \in \Ps(\Ps(\R^d) \times \Ps(\R^d))$, the optimizer $\Gs^*$ of~\eqref{MFG_oper_obj}, i.e., 
\begin{align*}
    \Gs^* \in \argmin_{\Gs} \E_{(P_0, P_1) \sim \mu} \Ls_s (P_0, P_1, \Gs(P_0,P_1))
\end{align*}
is the MFG solution operator up to measure zero sets of $\mu$, i.e.,
\begin{align*}
    \Gs^*(P_0,P_1) = F^* \in \argmin_{F} \Ls (P_0, P_1, F), \quad \mu\text{-a.e.},
\end{align*}
where $\mu$-a.e. means almost everywhere in the sense of $\mu$.
\end{thm}

\begin{remark}
By the manifold hypothesis~\cite{manifold_hypo, CAE}, realistic data typically lie on or near a low dimensional manifold embedded in a high dimensional ambient space. For example, $(P_0,P_1)$ may belong to certain classes of parametrized distributions. In this case, there exists local diffeomorphism $\phi: U\subseteq \R^{m_0} \times \R^{m_1} \to  W\subseteq \Ps(\R^d) \times \Ps(\R^d)$ that induces $\mu$ by $\mu \coloneqq \phi_*\hat{\mu}$, where $\hat{\mu} \in \Ps(\R^{m_0} \times \R^{m_1})$. Hence, if $\hat{\mu}$ is equivalent to the Lebesgue measure on $\R^{m_0} \times \R^{m_1}$, the optimizer $\Gs^*$ is the MFG solution map a.e. on the data manifold. 
\end{remark}

\begin{remark}
The proposed objective can be extended to enable solution operator learning of certain classes of PDEs without reliance on training demonstrations. The Poisson equation, for example, can be cast as a variational problem minimizing the Dirichlet energy~\cite{deep_ritz}, and solving its amortized counterpart yields the Poisson solution operator.
\end{remark}

Lastly, we study a special MFG to demonstrate that Monte Carlo estimations of certain optimal MFG solutions converges at a rate that only depends on the sample size. As such, our sampling based approach is scalable since we can use the same number of samples to represent $P_0, P_1$ regardless of their dimensionality.

\begin{prop}\label{prop:gaussian_opt_soln}
    Let $P_0 = \mathrm{N}(0,\sigma^2I), P_1 = \mathrm{N}(\vm, \sigma^2I)$, and $\Ms(P) = \mathrm{MMD}(P,P_1)$ with the linear kernel $k(x,y)=x\cdot y$. The interaction-free MFG
\begin{equation}\label{Gaussian_gaussian_MFG}
\begin{aligned}
    \inf_{T} \quad \int_{\R^d}  \| T(\vx) - \vx\|_2^2 &p_0(\vx)\der\vx + \lambda \mathcal{M}(T_*P_0) \\
\end{aligned}
\end{equation}
has the optimal solution $T^*(\vx) = \vx + \frac{\lambda}{1+\lambda} \vm $ and the optimal value $\frac{\lambda}{1+\lambda}\|\vm\|^2_2$.
    
\end{prop}

However, since $P_1$ is represented by its finite samples, we do not know $m$ directly and must estimate it. Thus, we expect a estimation error that diminishes as the number of available samples increase. Below, we provide a precise quantification of this error, characterizing the \textit{statistically} optimal MFG value for our setting. With samples ${\vx_i} \overset{\mathrm{iid}}{\sim} P_1$, a straightforward estimation is: $\Bar{T}(\vx) = \vx + \frac{\lambda}{1+\lambda}\Bar{\vx}$, where $\Bar{\vx} = \frac{1}{n}\sum_{i=1}^n \vx_i$ is the sample average. The following result characterizes the discrepancy between estimation $\Bar{T}$ and the optimal $T^*$.

\begin{prop}\label{prop:gaussian_Monge_est_err}
    Under the same assumption in Proposition \ref{prop:gaussian_opt_soln}, the relative $L_2$ error between the optimal Monge map $T^*(\vx) = \vx + \frac{\lambda}{1+\lambda} \vm $ and its finite sample estimation $\Bar{T}(\vx) = \vx + \frac{\lambda}{1+\lambda}\Bar{\vx}$, where $\Bar{\vx} = \frac{1}{n}\sum_{i=1}^n \vx_i, {\vx_i} \overset{\mathrm{iid}}{\sim} P_1, \forall i$ is
\begin{equation}\label{eqn:gaussian_Monge_est_err}
\begin{split}
    R_T &\coloneqq \frac{\E_{\vx \sim P_0} \|\Bar{T}(\vx) - T^*(\vx)\|^2_2}{\E_{\vx\sim P_0} \|T^*(\vx)\|^2_2 } = \frac{(\frac{\lambda}{1+\lambda})^2 d\sigma^2}{n[d\sigma^2 + (1+ \frac{\lambda}{1+\lambda})^2\|\vm\|^2_2]}
\end{split}
\end{equation}
    
\end{prop}


Assuming $\|\vm\|^2_2 = \mathcal{O}(d)$, for example, $\vm_i = \Os(1), \forall i$. The estimation errors vanishes linearly as $n \to \infty$ and has no dependence on $d$. Therefore, there is at least one estimate of the MFG solution that converges to the true solution at a rate independent of dimensionality.

In general, we cannot hope to analytically solve MFGs, so it is difficult to generalize the above analysis to a broader setting. However, we will show through numerical results that our approach's scalability with regard to $d$ may hold for a much larger class of problems.

\section{Numerical Experiments}
\label{sec:results}
We conduct comprehensive experiments on both synthetic and realistic datasets. Beginning with synthetic examples designed to be generalizable to different dimensions, we showcase the effectiveness of the proposed approach in solving MFG across varied complexities and dimensionalities. We then apply our method to solving MFG between MNIST digits to show its promise in realistic applications. Lastly, we compare with current single-instance neural solvers to show our method's comparable performance and much faster inference time. 

\subsection{Gaussian}\label{sect:gaussian}

\begin{figure}[h]
\centering

\begin{minipage}{1\linewidth}
\centering
\includegraphics[width=1\linewidth]{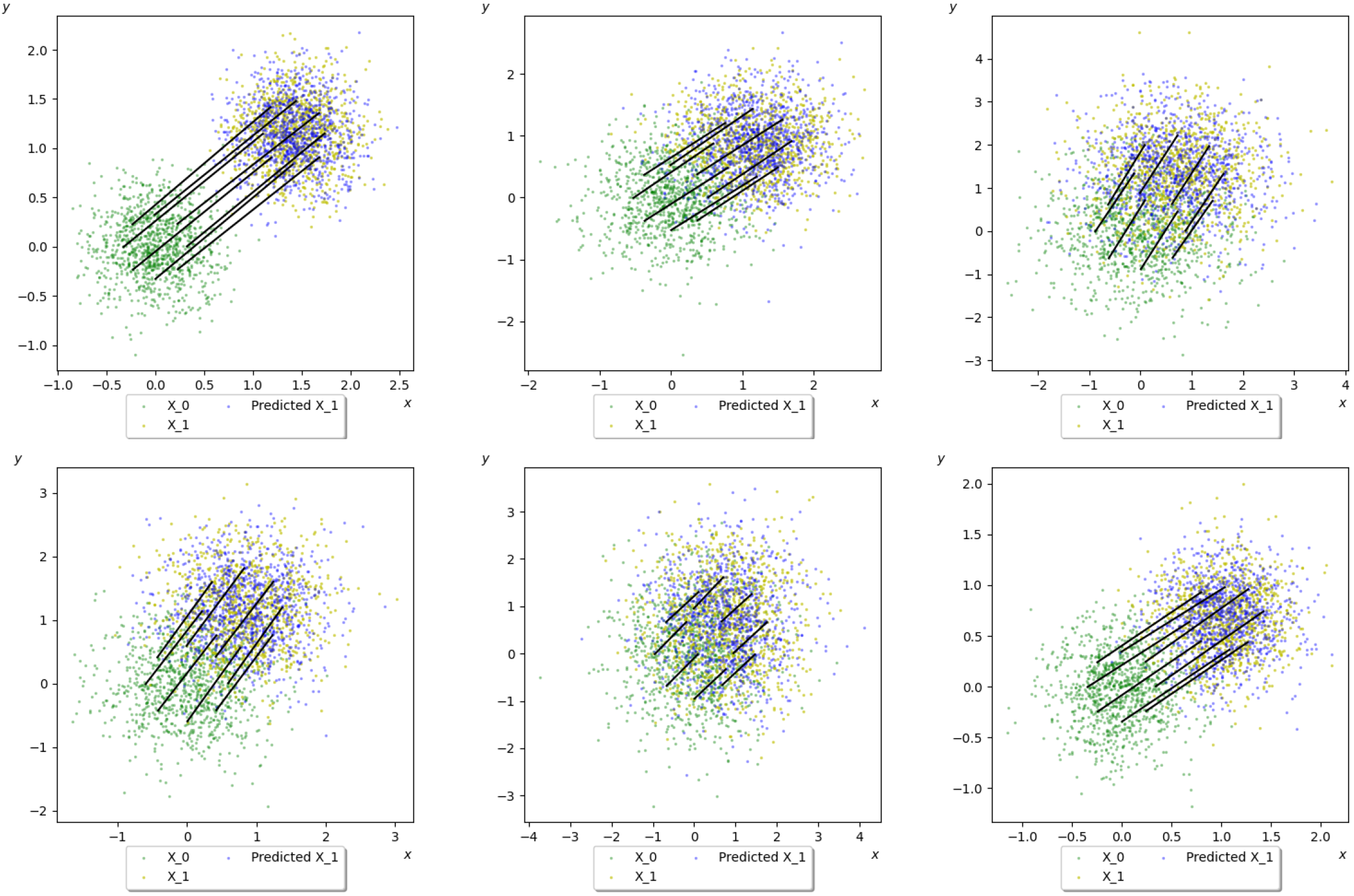}
\end{minipage}\hfill



\caption{Learned MFG trajectories for Gaussians in 10 (top) and 20 (bottom) dimensions projected to the first two components. Each row shows 3 MFG instances. The green, yellow, and blue dots represent samples from $P_0, P_1,$ and $G_\theta(P_0,P_1)_*P_0$, respectively. The black lines are learned trajectories for 8 selected landmarks.}
\label{fig:gaussian_gaussian}
\end{figure}

We start with an interaction-free MFG where the set of initial and terminal distributions are both Gaussian. Let the distribution of MFG configurations be the product measure $\mu = \mu_0 \otimes \mu_1$. Denote $U[a,b]^d$ as the uniform distribution on $[a,b]^d$. We obtain initial Gaussians $P_0 \sim \mu_0$ as $P_0 = \mathrm{N}(0, \sigma^2I)$ and terminal Gaussians $P_1 \sim \mu_1$ as $P_1 = \mathrm{N}(\vm, \sigma^2I)$, where $\vm\sim U[0.5,1.5]^d, \sigma^2 = 0.1 + 0.9a^2, a \sim U[0,1]$. 

Since each pair of initial and terminal distributions are Gaussians differing only in the mean, it suffices to use a linear kernel $k(x,y) = x\cdot y$ in the MMD, and the terminal cost is reduced to directly comparing the $l_2$ distance between means. The simplicity of this setup makes it possible to compare our numerical results to its closed form solution discussed in Section~\ref{sect:theo_results} to concretely demonstrate our method's effectiveness. In addition, we conduct a comparative study for various architectures that are also permutation and sampling invariant in Table~\ref{tab:gaussian_results}. It is evident that transformers outperform other baselines in all cases, so we use it as our backbone architecture for all experiments.

\begin{table}[h]
\begin{center}
\begin{small}
\begin{sc}
\begin{tabular}{lccccccccccc}
\toprule
Method & $d=2$ & $d=5$ & $d=10$ & $d=20$ & $d=100$ & Parameters\\
\midrule
Pointwise MLP & 0.3807 & 0.2058 & 0.1432 & 0.1001 & 0.0673 & 27195020 \\
Spatial Convolution~\cite{PointNet++} & 0.4867 & 0.3646 & 0.3387 & 0.3106 & 0.3151 & 47925268 \\
Folding Layer~\cite{foldingnet} & 0.0318 & 0.0184 & 0.0178 & 0.0316 & 0.0228 & 22714772 \\
Ours & \textbf{0.0309} & \textbf{0.0178} & \textbf{0.0136} & \textbf{0.0097} & \textbf{0.0056} & 29390850 \\
\midrule
Theoretical Optimum & 0.0309 & 0.0178 & 0.0136 & 0.0096 & 0.0055 & - \\
\bottomrule
\end{tabular}
\end{sc}
\end{small}
\end{center}
\caption{Relative $L_2$ errors for the Gaussian setting in different dimensions($d$) with $n=1024$. The theoretical optima are computed via averaging Equation~\ref{eqn:gaussian_Monge_est_err} over $\vm\sim U[0.5,1.5]^d, \sigma^2 = 0.1 + 0.9a^2, a \sim U[0,1]$. The last column shows the number of trainable parameters in a model.}
\label{tab:gaussian_results}
\vskip -10pt
\end{table}

In Figure~\ref{fig:gaussian_gaussian}, we provide qualitative examples for the learned OT trajectories in various dimensions. Consistent with our expectations, the learned trajectories are visually parallel, indicating the learned couplings are approximately translations on the entire density. Additionally, we plot the normalized relative mean absolute error (MAE) between the computed MFG cost and the true MFG cost, then compare this value to the best estimated MFG cost from finite samples. As shown in the left picture in Figure~\ref{fig:err_and_loss}, the computed MFG cost's relative error coincides with the statistically optimal relative error as training progresses, suggesting that our model has learned to extract all salient information from the samples.

\subsection{Gaussian Mixture}~\label{sect:gaussian_mix}

\begin{figure}[h]
\centering
\vskip -15pt


\begin{minipage}{1\linewidth}
\centering
\includegraphics[width=1\linewidth]{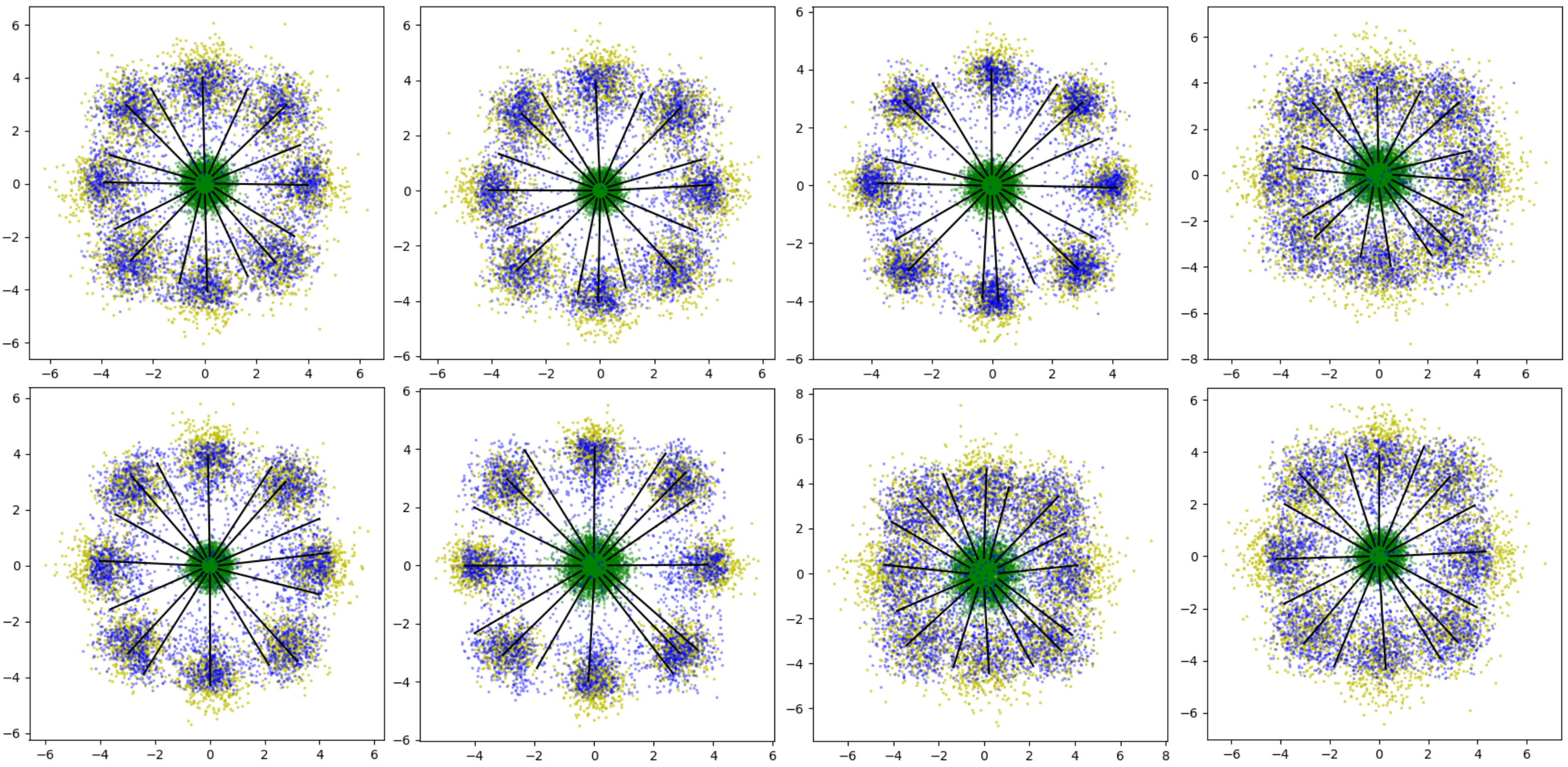}
\end{minipage}\hfill

\caption{Learned MFG trajectories for Gaussian mixture in 10(top) and 20(bottom) dimensions projected to the first two components. Each row shows 4 MFG instances. The green, yellow, and blue dots represent samples from $P_0, P_1,$ and $G_\theta(P_0,P_1)_*P_0$, respectively. The black lines represent learned trajectories for 16 selected landmarks.}
\label{fig:gaussian_mixture}
\end{figure}

Generalizing the setup in~\cite{NN_MFP}, we consider a MFG that seeks to transport a set of Gaussians to Gaussian mixtures with minimal movement and no interaction costs. Denote the distribution of configurations as $\mu = \mu_0 \otimes \mu_1$. A sample $P_0 \sim \mu_0$ is obtained by first sampling $\sigma^2 = 0.1 + 0.7a^2, a \sim U[0,1]$, then constructing $P_0 = \mathrm{N}(0, \sigma^2 I)$. For the terminal distribution, we sample $P_1 \sim \mu_1$ by first getting $s^2 \sim U[0.1,0.8]$, then constructing the mixture $P_1 = \frac{1}{8}\sum_{i=1}^8 \mathrm{N}(\mu_i, s^2 I)$, where $\mu_i = 4\cos(\frac{\pi}{4}i)e_1 + 4\sin(\frac{\pi}{4}i)e_2, \forall i=1, ..., 8$ are fixed, and $e_1, e_2$ are the first two standard basis vectors in $\R^d$. 

Visually, the MFG problem seeks to map between a single Gaussian at the origin and eight identical Gaussians equidistantly placed on the circle of radius 4. Therefore, the initial density needs to change its shape during the course of travel, forcing the learned operator to capture couplings beyond translations. In particular, we can no longer use the linear kernel as each pair of $P_0, P_1$ have identical means. Instead, we compute MMD with the Laplacian kernel, which can discriminate any two distributions~\cite{MMD_kernel_test}.

\begin{figure}[h]
\centering
\begin{minipage}{0.33\linewidth}
\centering
\includegraphics[width=1\linewidth]{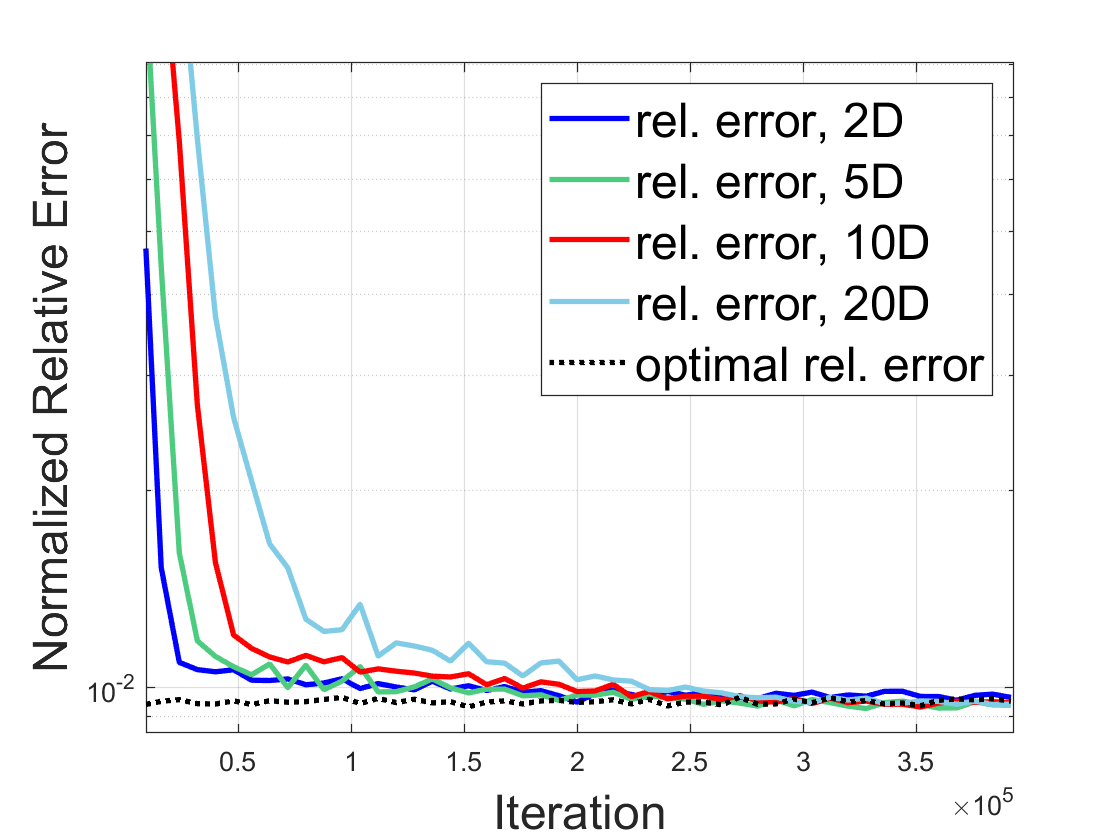}
\end{minipage}\hfill
\begin{minipage}{0.33\linewidth}
\centering
\includegraphics[width=1\linewidth]{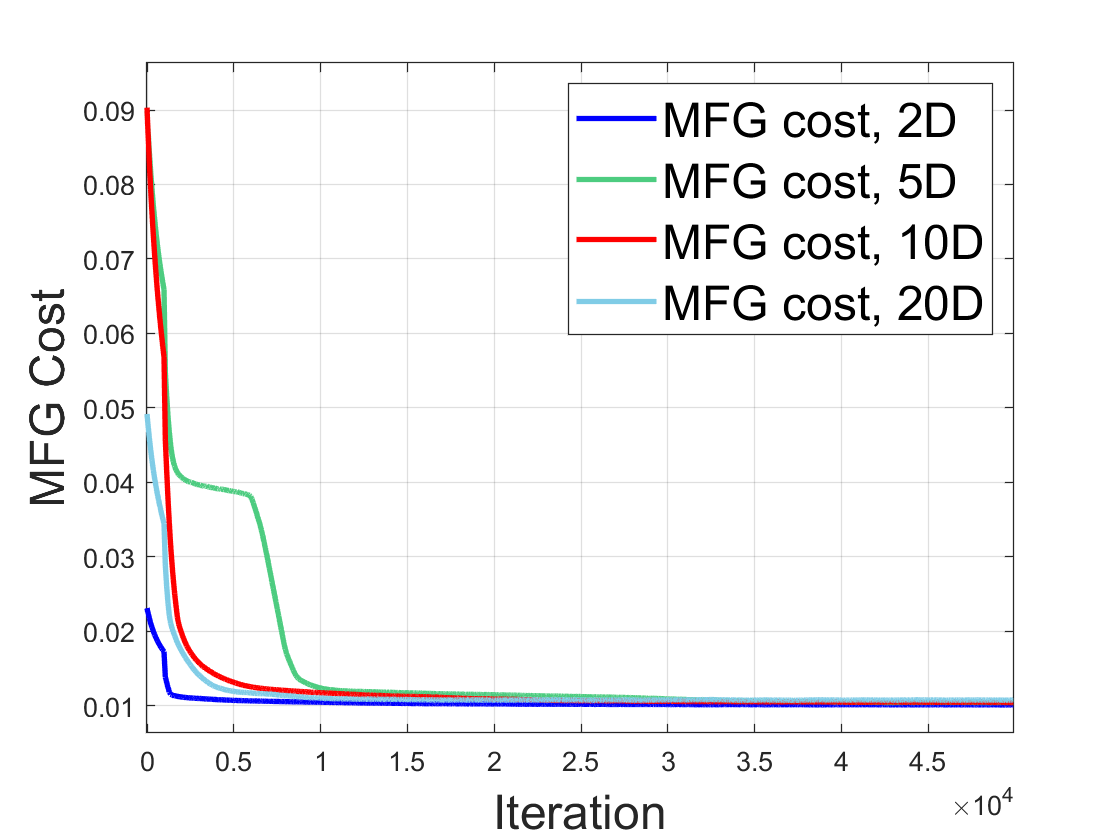}
\end{minipage}\hfill
\begin{minipage}{0.33\linewidth}
\centering
\includegraphics[width=1\linewidth]{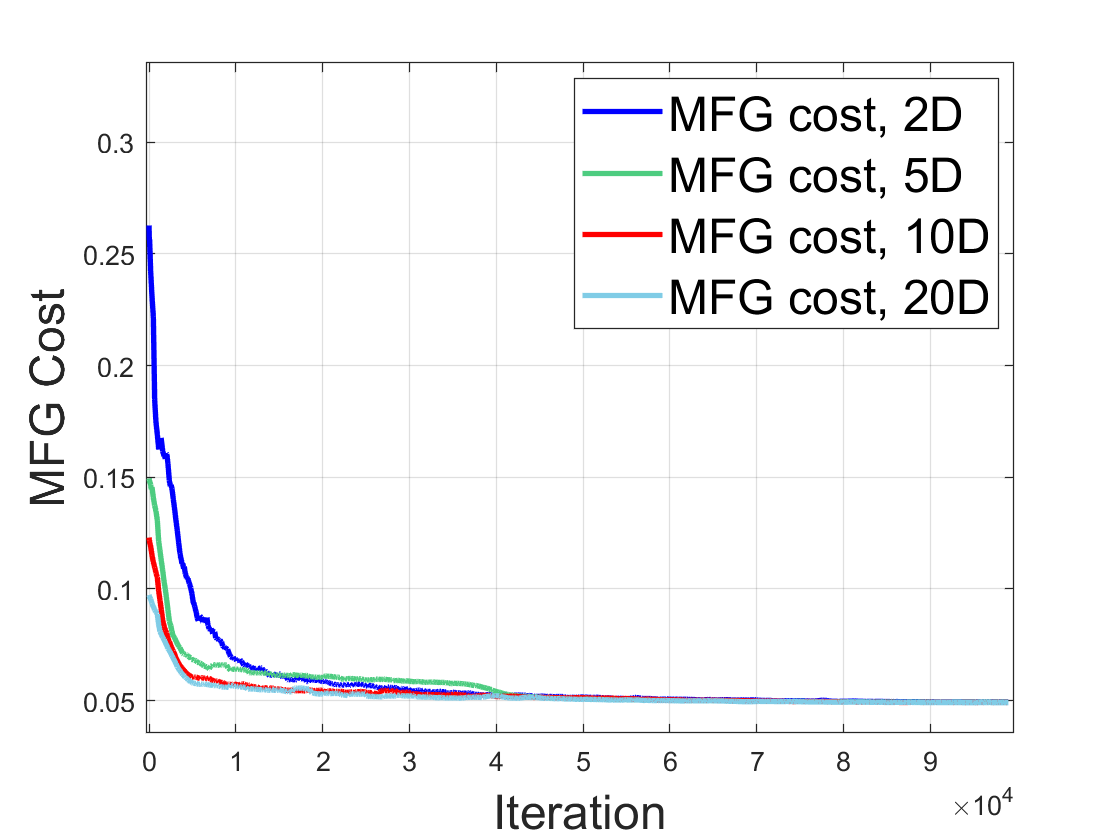}
\end{minipage}\hfill
\caption{Left: Normalized relative $L_1$ errors on the MFG value for Gaussians in $d=2,5,10,20$ with sample size 1024 compared to the statistically optimal MFG value. All errors are normalized by $\sqrt{\frac{d}{d_{max}}} = \sqrt{\frac{d}{20}}$ so that the optimal values are identical across dimensions. Middle: Learned MFG costs for Gaussian mixture in $d=2,5,10,20$. Right: Learned MFG costs for crowd motion in $d=2,5,10,20$. }
\label{fig:err_and_loss}
\end{figure}

Although the closed form solution is not known in this case, the interaction-free MFG is still equivalent to a related OT problem. As such, we expect the optimal agent trajectories to be straight and non-intersecting. In addition, the experimental setup ensures the optimal trajectories are invariant to problem dimension when projected down to the first two basis components, and the problems share the same optimal MFG cost. Intuitively, the most efficient transportation plan partitions the center Gaussian into eight components of equal mass, then transports each component radially outward while morphing it into a cluster in the mixture. From Figure~\ref{fig:gaussian_mixture}, we observe this behavior faithfully in different dimensions across various sampled configurations. 

Furthermore, we plot the evolution of computed MFG costs for different dimensions in Figure~\ref{fig:err_and_loss}, noting that their values converge to the same minima, corroborating our expectation. Remarkably, models in different dimensions are trained with the identical sample size $n=1024$, highlighting our method's scalability with respect to $d$. 

\begin{figure}[h]
\centering

\vskip -10pt
\begin{minipage}{0.16\linewidth}
\centering
\includegraphics[width=1\linewidth]{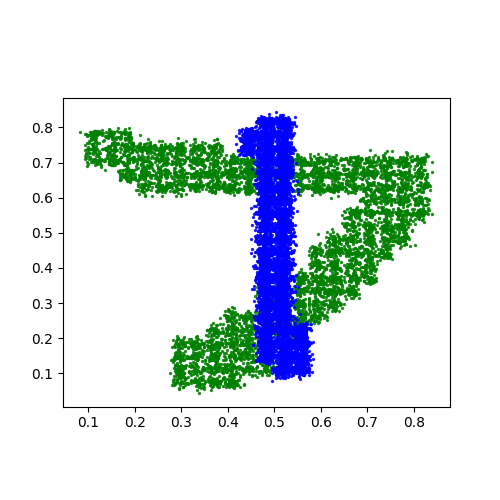}
\end{minipage}\hfill
\begin{minipage}{0.16\linewidth}
\centering
\includegraphics[width=1\linewidth]{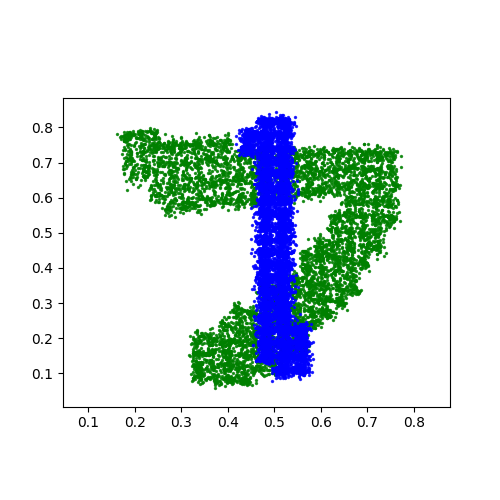}
\end{minipage}\hfill
\begin{minipage}{0.16\linewidth}
\centering
\includegraphics[width=1\linewidth]{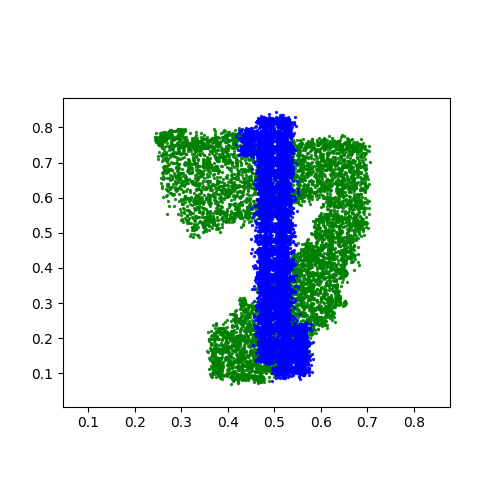}
\end{minipage}\hfill
\begin{minipage}{0.16\linewidth}
\centering
\includegraphics[width=1\linewidth]{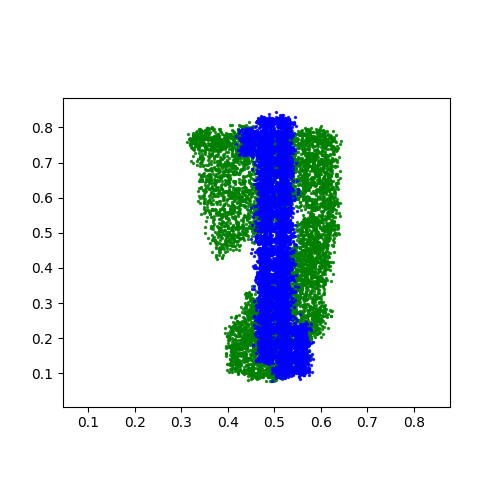}
\end{minipage}\hfill
\begin{minipage}{0.16\linewidth}
\centering
\includegraphics[width=1\linewidth]{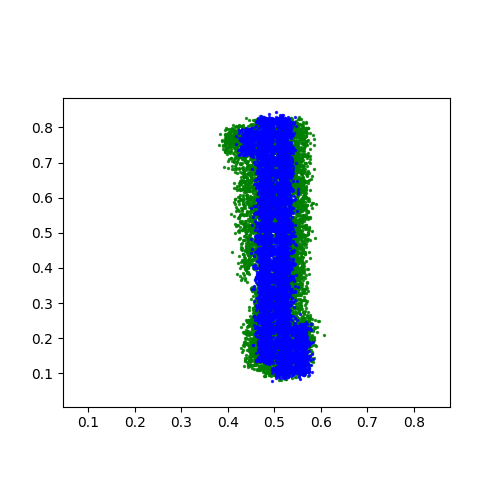}
\end{minipage}\hfill
\begin{minipage}{0.16\linewidth}
\centering
\includegraphics[width=1\linewidth]{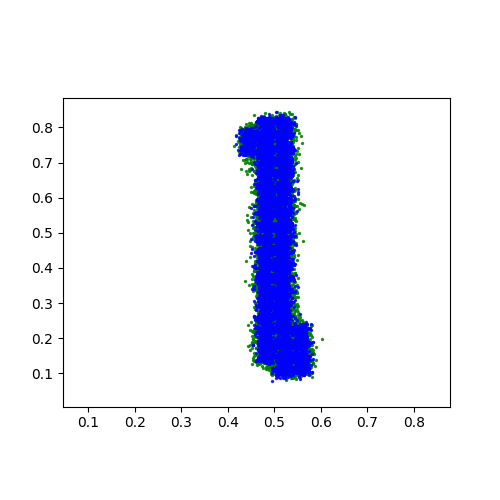}
\end{minipage}\hfill

\vskip -15pt
\begin{minipage}{0.16\linewidth}
\centering
\includegraphics[width=1\linewidth]{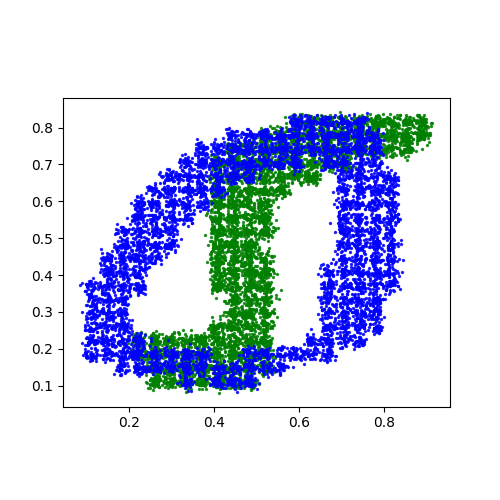}
\end{minipage}\hfill
\begin{minipage}{0.16\linewidth}
\centering
\includegraphics[width=1\linewidth]{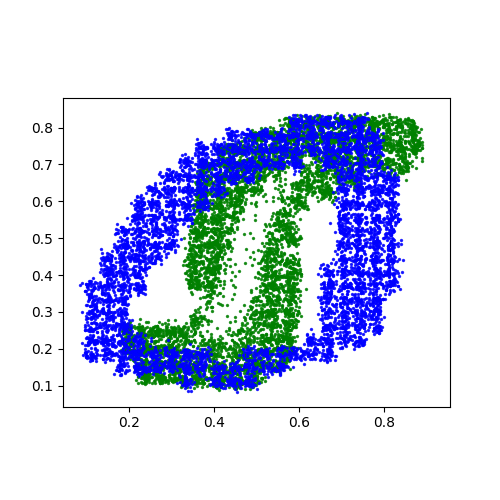}
\end{minipage}\hfill
\begin{minipage}{0.16\linewidth}
\centering
\includegraphics[width=1\linewidth]{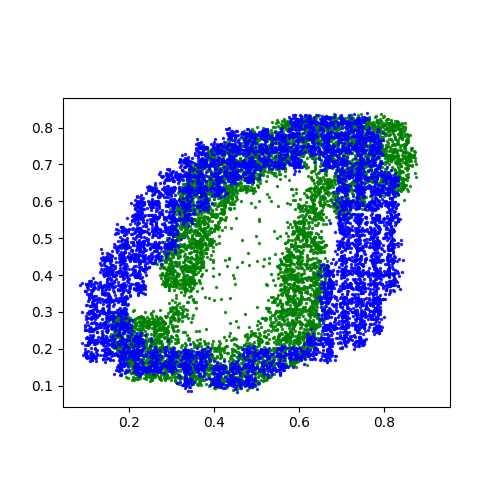}
\end{minipage}\hfill
\begin{minipage}{0.16\linewidth}
\centering
\includegraphics[width=1\linewidth]{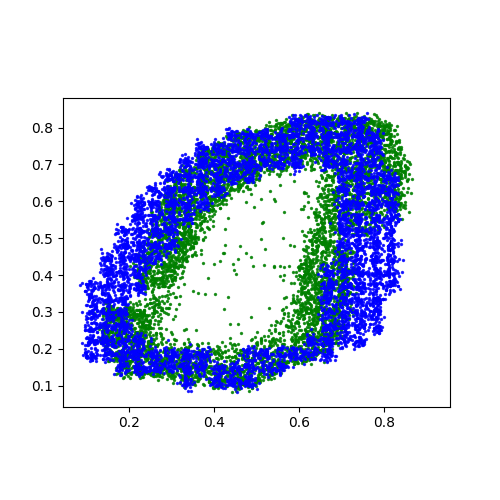}
\end{minipage}\hfill
\begin{minipage}{0.16\linewidth}
\centering
\includegraphics[width=1\linewidth]{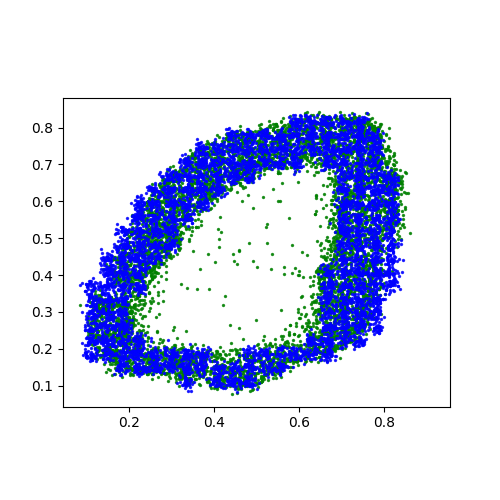}
\end{minipage}\hfill
\begin{minipage}{0.16\linewidth}
\centering
\includegraphics[width=1\linewidth]{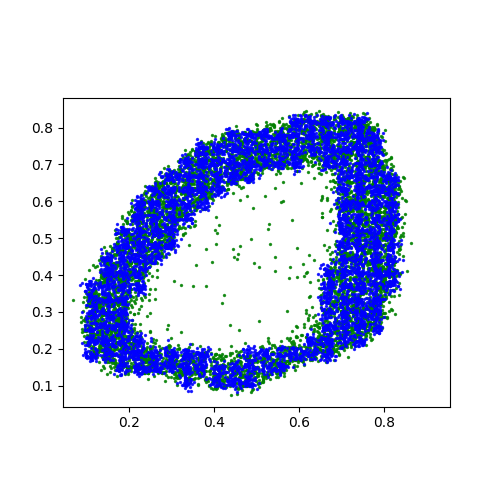}
\end{minipage}\hfill

\vskip -15pt
\begin{minipage}{0.16\linewidth}
\centering
\includegraphics[width=1\linewidth]{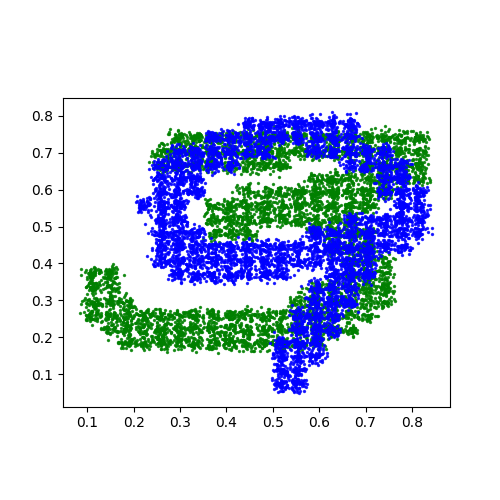}
\end{minipage}\hfill
\begin{minipage}{0.16\linewidth}
\centering
\includegraphics[width=1\linewidth]{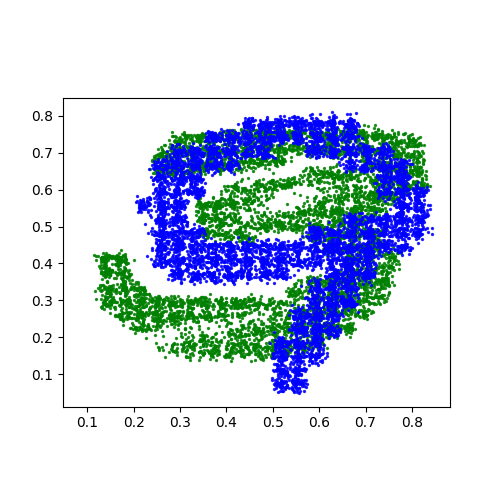}
\end{minipage}\hfill
\begin{minipage}{0.16\linewidth}
\centering
\includegraphics[width=1\linewidth]{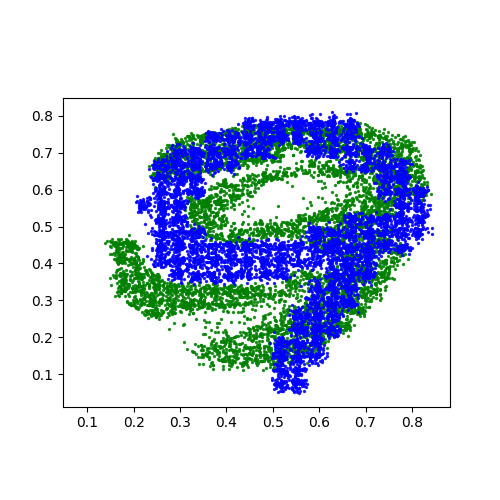}
\end{minipage}\hfill
\begin{minipage}{0.16\linewidth}
\centering
\includegraphics[width=1\linewidth]{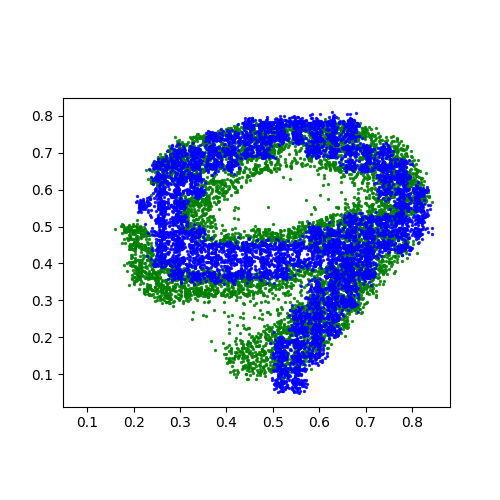}
\end{minipage}\hfill
\begin{minipage}{0.16\linewidth}
\centering
\includegraphics[width=1\linewidth]{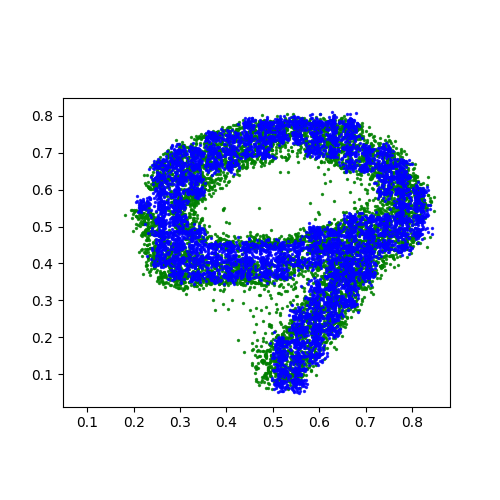}
\end{minipage}\hfill
\begin{minipage}{0.16\linewidth}
\centering
\includegraphics[width=1\linewidth]{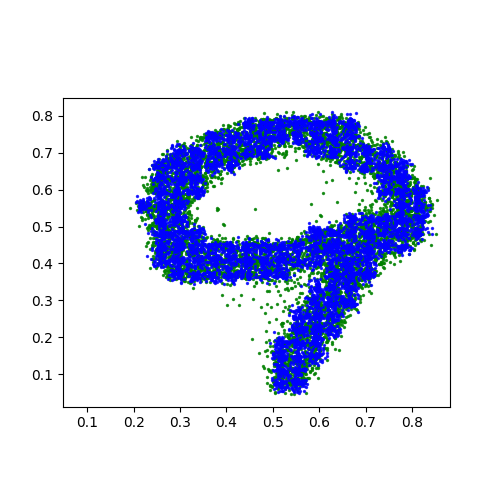}
\end{minipage}\hfill

\vskip -15pt
\begin{minipage}{0.16\linewidth}
\centering
\includegraphics[width=1\linewidth]{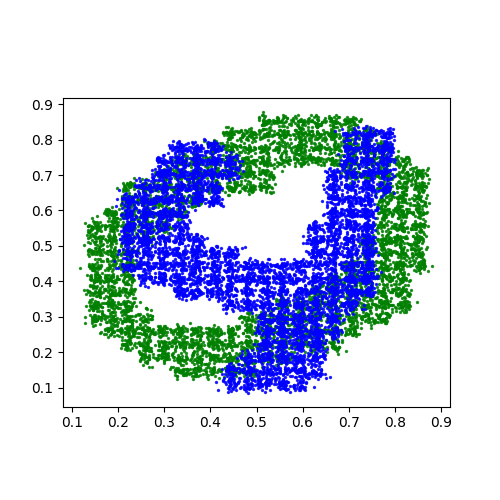}
\end{minipage}\hfill
\begin{minipage}{0.16\linewidth}
\centering
\includegraphics[width=1\linewidth]{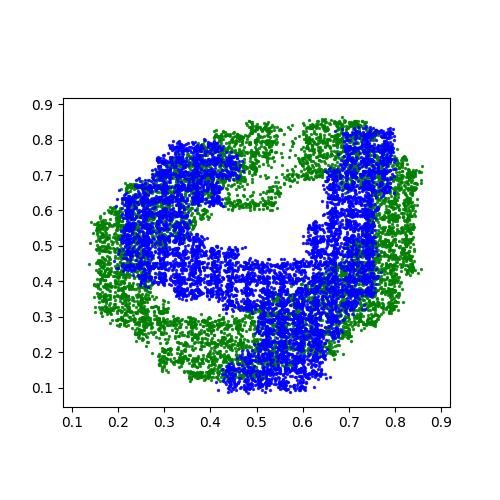}
\end{minipage}\hfill
\begin{minipage}{0.16\linewidth}
\centering
\includegraphics[width=1\linewidth]{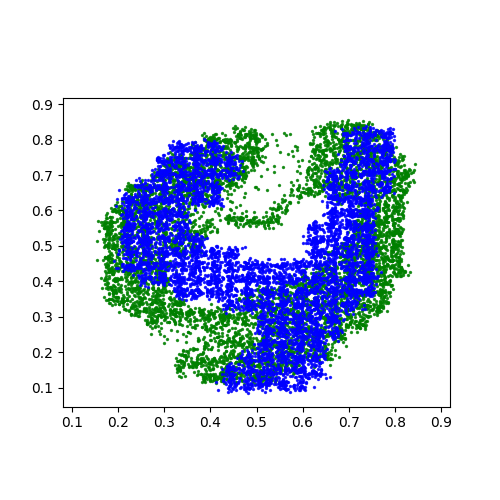}
\end{minipage}\hfill
\begin{minipage}{0.16\linewidth}
\centering
\includegraphics[width=1\linewidth]{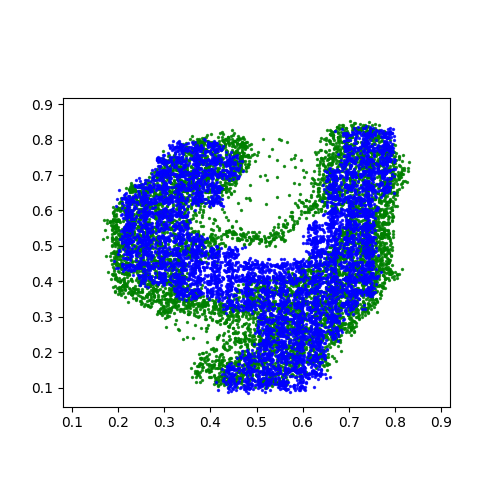}
\end{minipage}\hfill
\begin{minipage}{0.16\linewidth}
\centering
\includegraphics[width=1\linewidth]{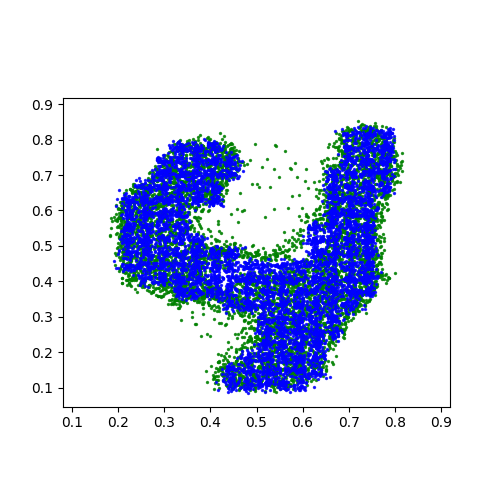}
\end{minipage}\hfill
\begin{minipage}{0.16\linewidth}
\centering
\includegraphics[width=1\linewidth]{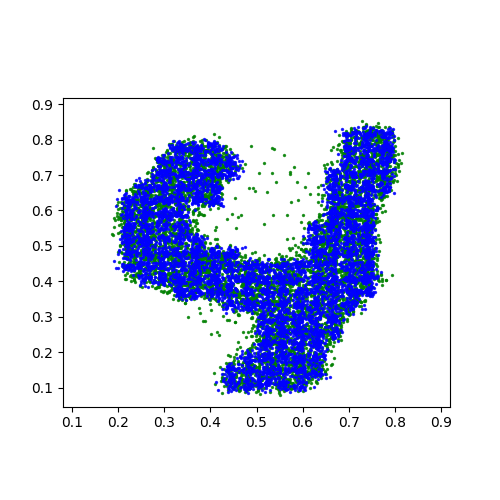}
\end{minipage}\hfill

\caption{Learned couplings between MNIST digits from the testing set, each represented by $n=7020$ samples. Each row shows the transportation between $P_0$ (green dots) and $P_1$ (blue dots) upon solving an interaction-free MFG.}
\label{fig:MNIST}
\vskip -5pt
\end{figure}

\subsection{MNIST}

We apply our operator learning framework to learn MFG couplings between MNIST images. Taking the distribution of configurations $\mu$ as handwritten digits from 0 to 9, we seek to solve interaction-free MFG between $P_0, P_1$ sampled iid from $\mu$ with the $L_2$ transport cost. As a result, our model learns to output OT-like mappings between any two digits.

As our framework adopts a sampling-based representation of $P_0, P_1$, we opt for a point-cloud version of MNIST images~\cite{pcMNIST}, each containing 1053 points dequantized with small Gaussian noises. We use the common split of 60k training and 10k testing digits. Compared to synthetic experiments, variations between MNIST samples are more meaningful and complex while our access to data is limited, making this case significantly more difficult.

In Figure~\ref{fig:MNIST}, we display learned evolutions between various digits from the testing set to illustrate the efficacy of our approach. For all sampled pairs, we observe that the learned operator transports $P_0$ to a shape that closely resembles $P_1$. Furthermore, the coupling visually traverses the Wasserstein geodesic and incurs a small mean distance traveled for all particles, as evidenced by e.g., partitioning a digit five into two halves and transporting each outwards into a digit zero on the second row of Figure~\ref{fig:MNIST}.

In addition, Figure~\ref{fig:MNIST_with_metaOT} compares our learned couplings to those from Meta-OT~\cite{meta_OT}, which computes the OT transportation between MNIST digits. Qualitatively, our results bears strong visual similarity to the Meta-OT mappings, so the two approaches are equally capable on this task. However, it is worth noting that as we model the problem as an interaction-free MFG, its solution agrees to the dynamic optimal transport problem only when $\lambda_L \to 0$ ($\lambda_L = 2\cdot 10^{-2}$ here). Moreover, we treat the transportation as a continuous problem in the ground space $\R^2$, while Meta-OT frames it as a discrete problem on the pixel space $\R^{784}$.

\begin{figure}[h]
\centering
\vskip 10pt
\begin{minipage}{1\linewidth}
\centering
\includegraphics[width=1\linewidth]{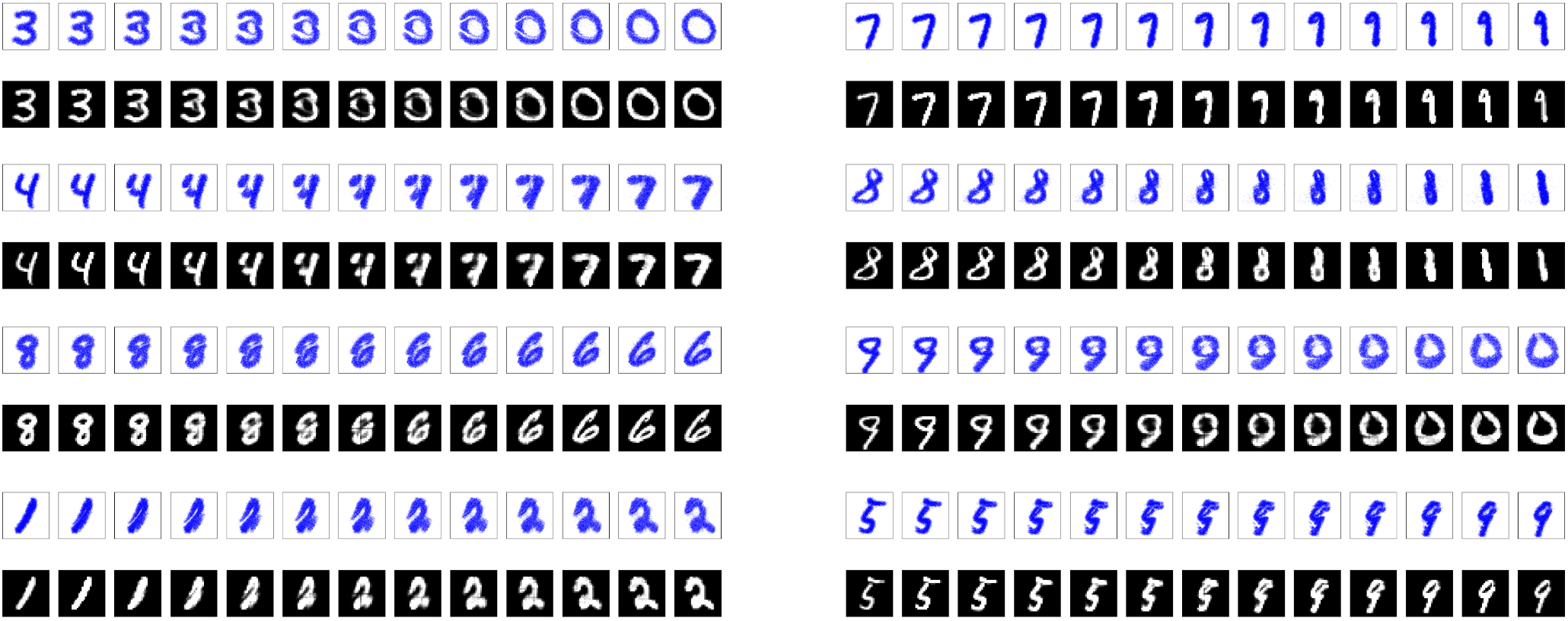}
\end{minipage}\hfill
\caption{Comparison between our (blue) learned transportation between MNIST digits and Meta-OT's (black and white).}
\label{fig:MNIST_with_metaOT}
\end{figure}

\subsection{Sampling Invariance}

\begin{figure}[h]
\centering

\vskip -10pt
\begin{minipage}{0.3\linewidth}
\centering
\includegraphics[width=1\linewidth]{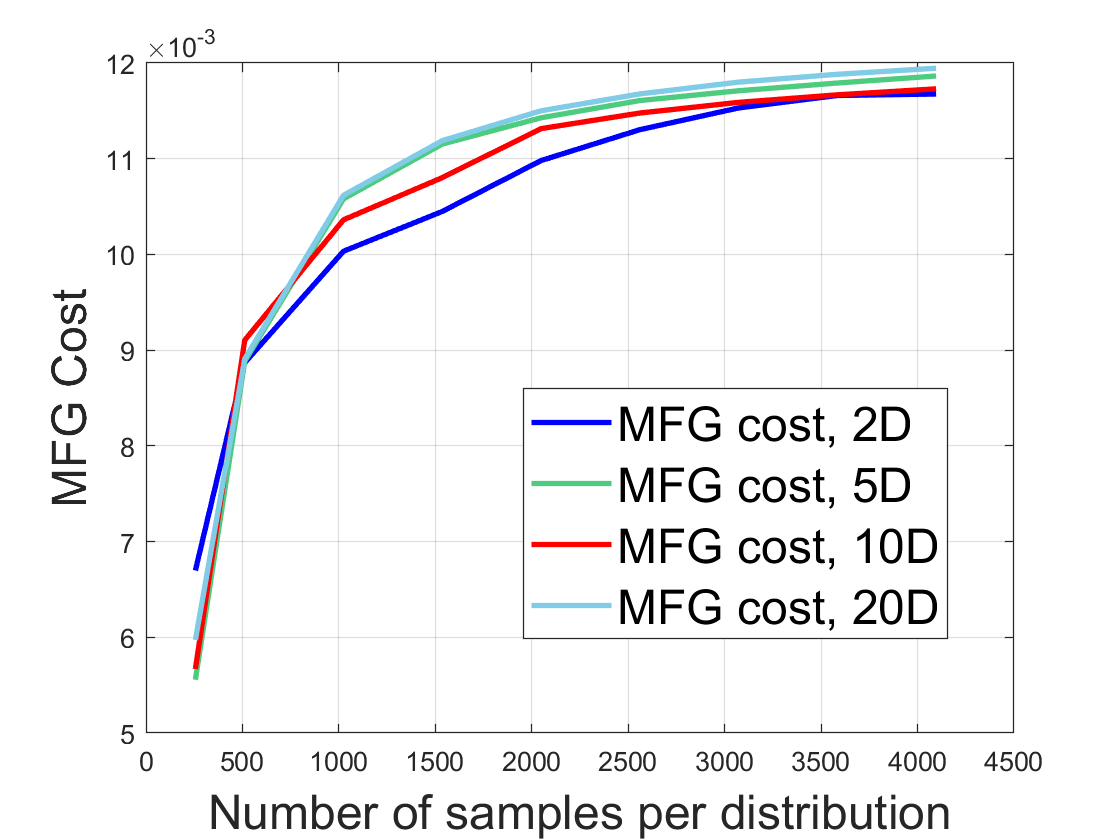}
\end{minipage}\hfill
\begin{minipage}{0.23\linewidth}
\centering
\includegraphics[width=1\linewidth]{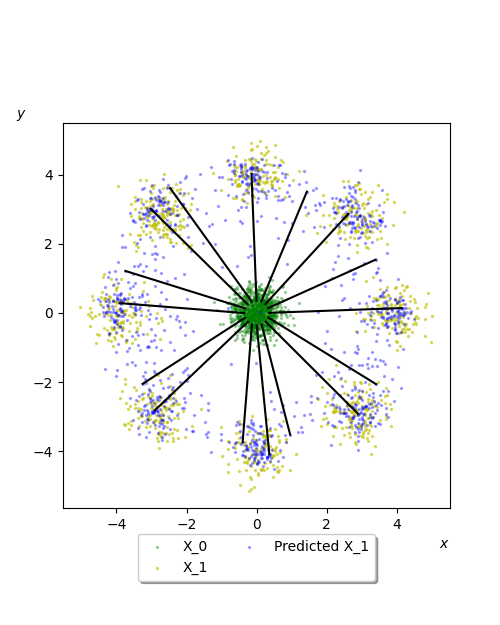}
\end{minipage}\hfill
\begin{minipage}{0.23\linewidth}
\centering
\includegraphics[width=1\linewidth]{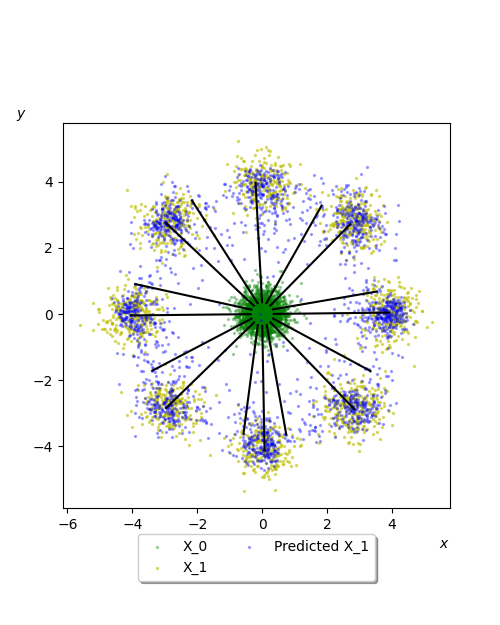}
\end{minipage}\hfill
\begin{minipage}{0.23\linewidth}
\centering
\includegraphics[width=1\linewidth]{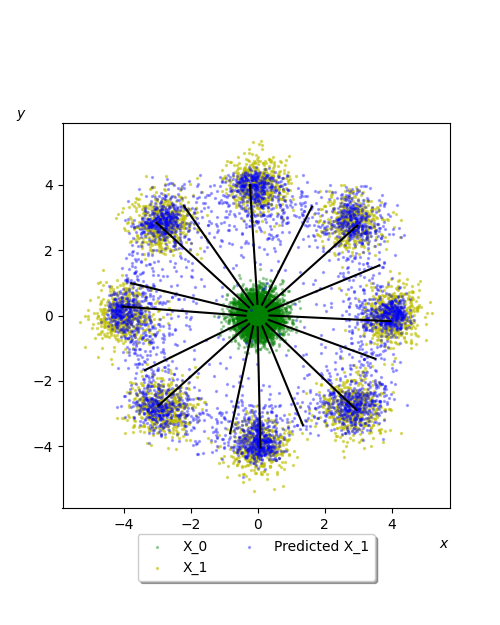}

\end{minipage}\hfill

\vskip -10pt
\begin{minipage}{0.16\linewidth}
\centering
\includegraphics[width=1\linewidth]{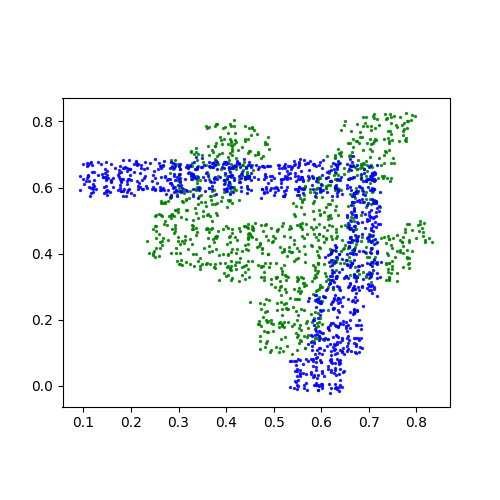}
\end{minipage}\hfill
\begin{minipage}{0.16\linewidth}
\centering
\includegraphics[width=1\linewidth]{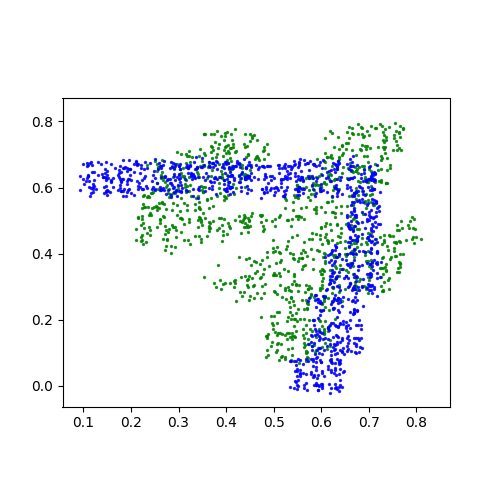}
\end{minipage}\hfill
\begin{minipage}{0.16\linewidth}
\centering
\includegraphics[width=1\linewidth]{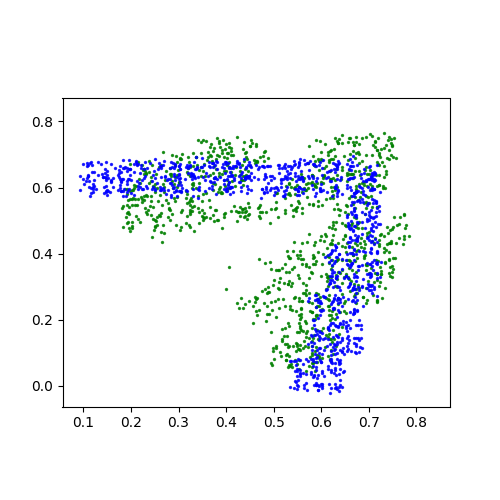}
\end{minipage}\hfill
\begin{minipage}{0.16\linewidth}
\centering
\includegraphics[width=1\linewidth]{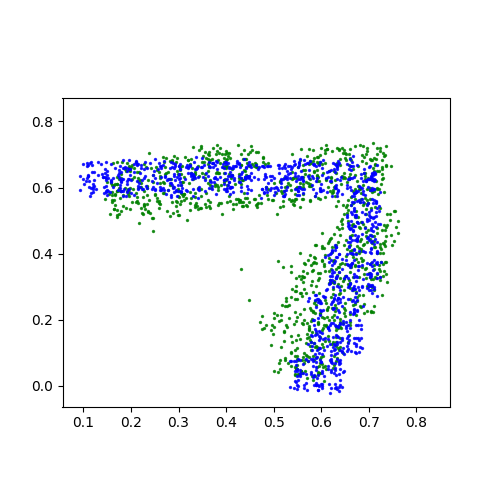}
\end{minipage}\hfill
\begin{minipage}{0.16\linewidth}
\centering
\includegraphics[width=1\linewidth]{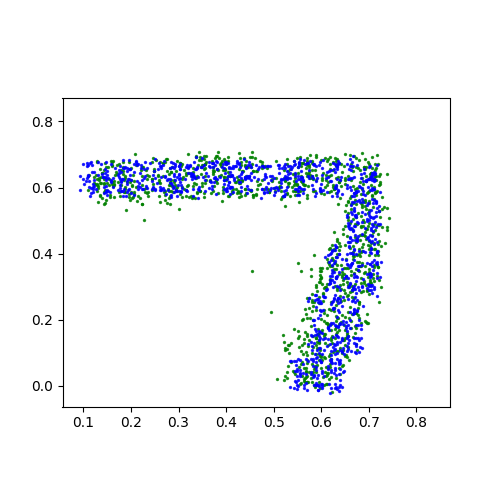}
\end{minipage}\hfill
\begin{minipage}{0.16\linewidth}
\centering
\includegraphics[width=1\linewidth]{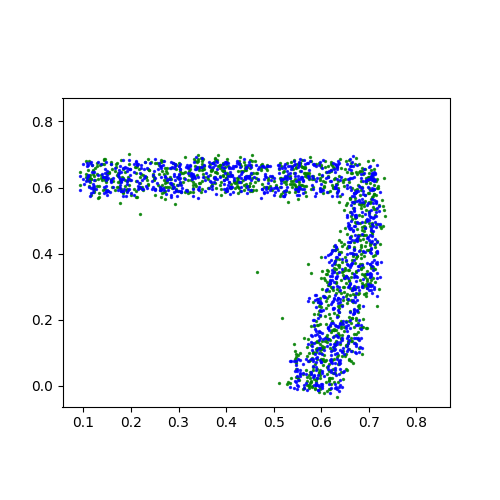}
\end{minipage}\hfill

\vskip -15pt
\begin{minipage}{0.16\linewidth}
\centering
\includegraphics[width=1\linewidth]{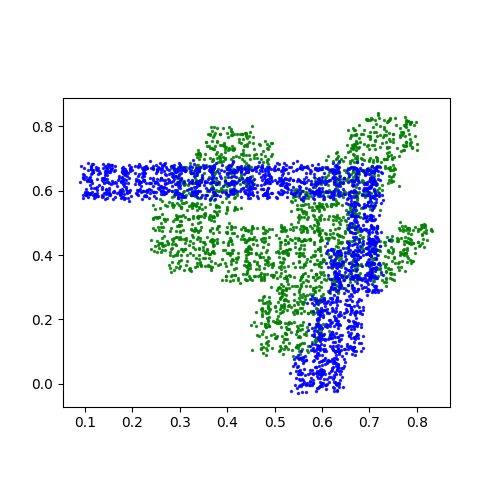}
\end{minipage}\hfill
\begin{minipage}{0.16\linewidth}
\centering
\includegraphics[width=1\linewidth]{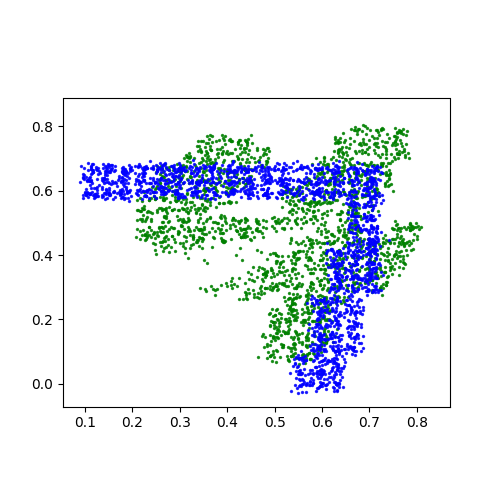}
\end{minipage}\hfill
\begin{minipage}{0.16\linewidth}
\centering
\includegraphics[width=1\linewidth]{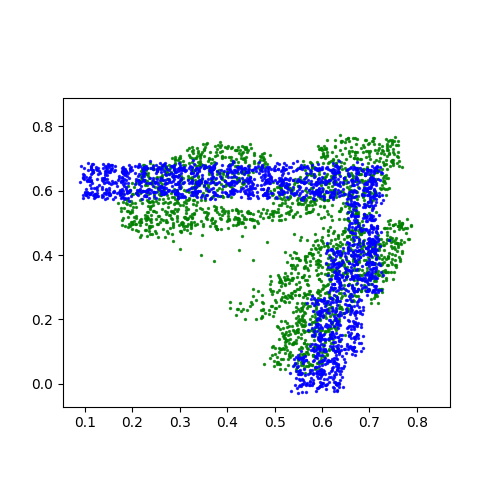}
\end{minipage}\hfill
\begin{minipage}{0.16\linewidth}
\centering
\includegraphics[width=1\linewidth]{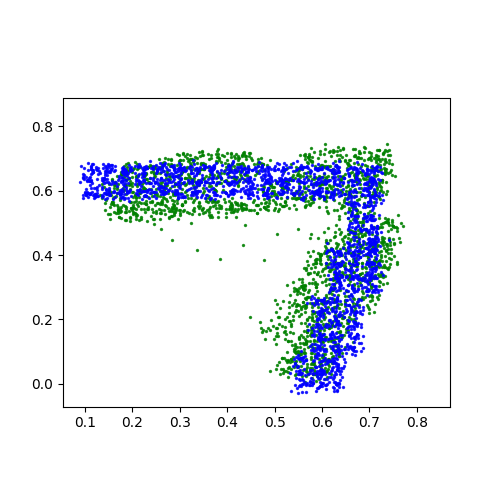}
\end{minipage}\hfill
\begin{minipage}{0.16\linewidth}
\centering
\includegraphics[width=1\linewidth]{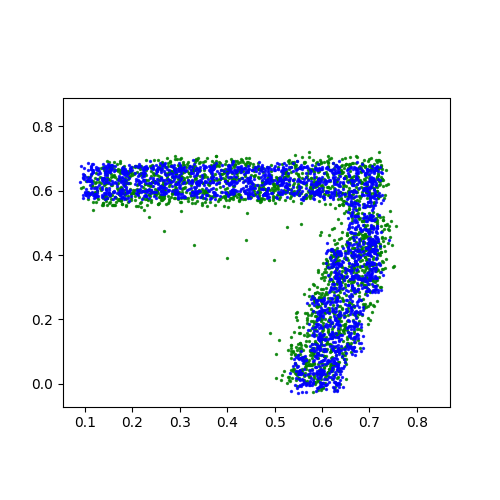}
\end{minipage}\hfill
\begin{minipage}{0.16\linewidth}
\centering
\includegraphics[width=1\linewidth]{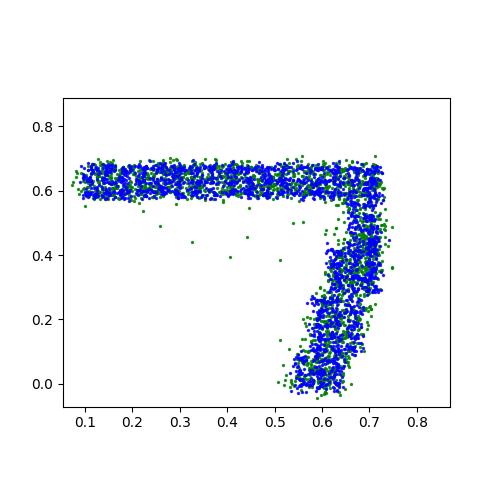}
\end{minipage}\hfill

\vskip -15pt
\begin{minipage}{0.16\linewidth}
\centering
\includegraphics[width=1\linewidth]{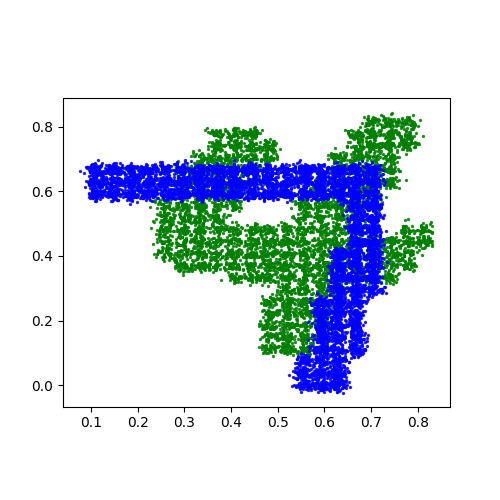}
\end{minipage}\hfill
\begin{minipage}{0.16\linewidth}
\centering
\includegraphics[width=1\linewidth]{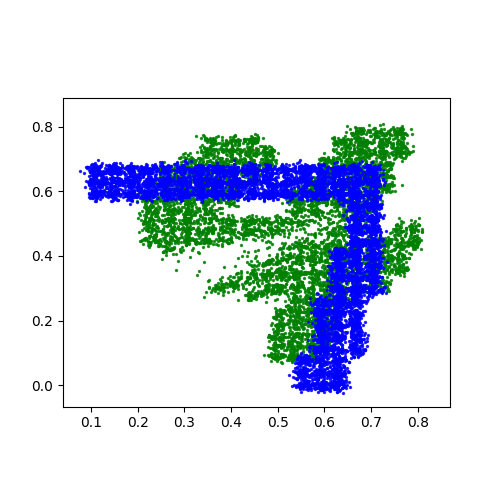}
\end{minipage}\hfill
\begin{minipage}{0.16\linewidth}
\centering
\includegraphics[width=1\linewidth]{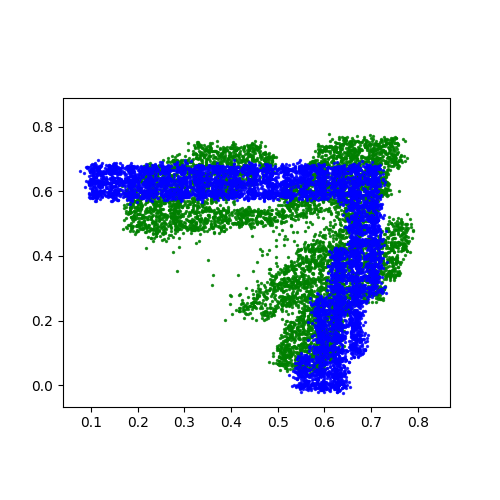}
\end{minipage}\hfill
\begin{minipage}{0.16\linewidth}
\centering
\includegraphics[width=1\linewidth]{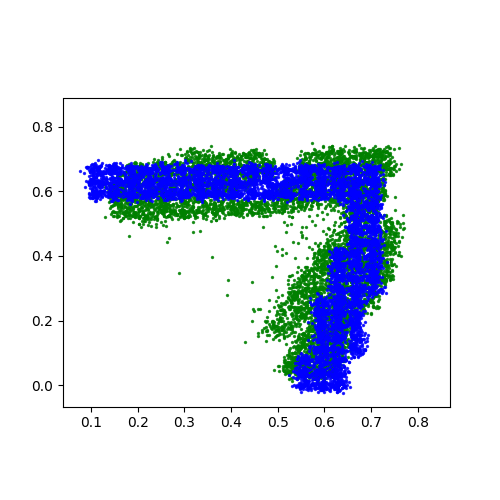}
\end{minipage}\hfill
\begin{minipage}{0.16\linewidth}
\centering
\includegraphics[width=1\linewidth]{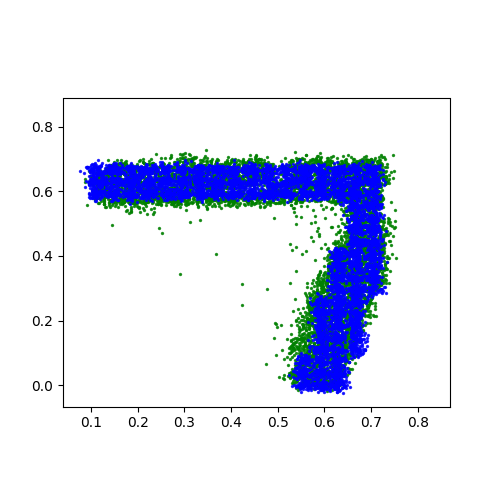}
\end{minipage}\hfill
\begin{minipage}{0.16\linewidth}
\centering
\includegraphics[width=1\linewidth]{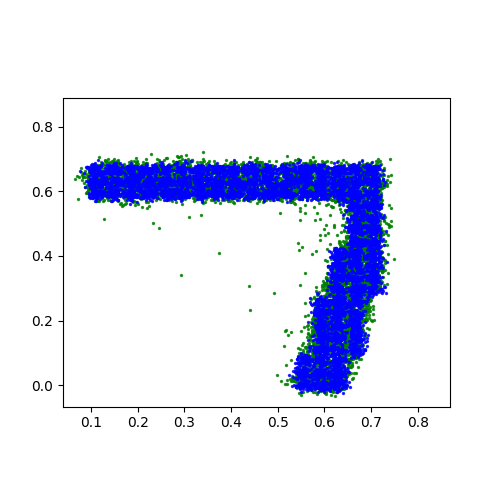}
\end{minipage}\hfill

\caption{Top row, left-most: MFG cost for models trained on Gaussian mixture with 1024 samples and evaluated on number of samples between 256 and 4096. Top row, right, 3 plots: Inferred couplings in $d=10$ with $n=1024, 2048, 4096$ samples. Second to last row: Inferred couplings between MNIST digits from the testing set with $n=1053, 2106, 7020$ samples.}
\label{fig:samping_inv}
\end{figure}

In this section, we numerically verify sampling-invariance of our model by studying a trained model's behavior on inputs of various sizes during inference. Once training is finished, the model undergoes no further weight updates, yet its sampling-invariant nature gives rise to consistent behaviors on different sample sizes.

The top left plot in Figure~\ref{fig:samping_inv} illustrates the computed MFG costs. The model underwent training using 1024 samples per distribution and is tested with sample sizes ranging from 256 to 4096. The procedure is repeated for dimensions 2, 5, 10, and 20. Across all dimensions, the MFG cost stabilizes as sample sizes increase, indicating the model's tendency toward a fixed mapping in the limit, affirming its sampling invariance. Furthermore, Figure~\ref{fig:samping_inv} includes qualitative results showcasing visual output consistency across different input sizes in the Gaussian mixture and the MNIST dataset.

\subsection{Crowd Motion}

\begin{figure}[h]
\centering

\vskip -15pt

\includegraphics[width=1\linewidth]{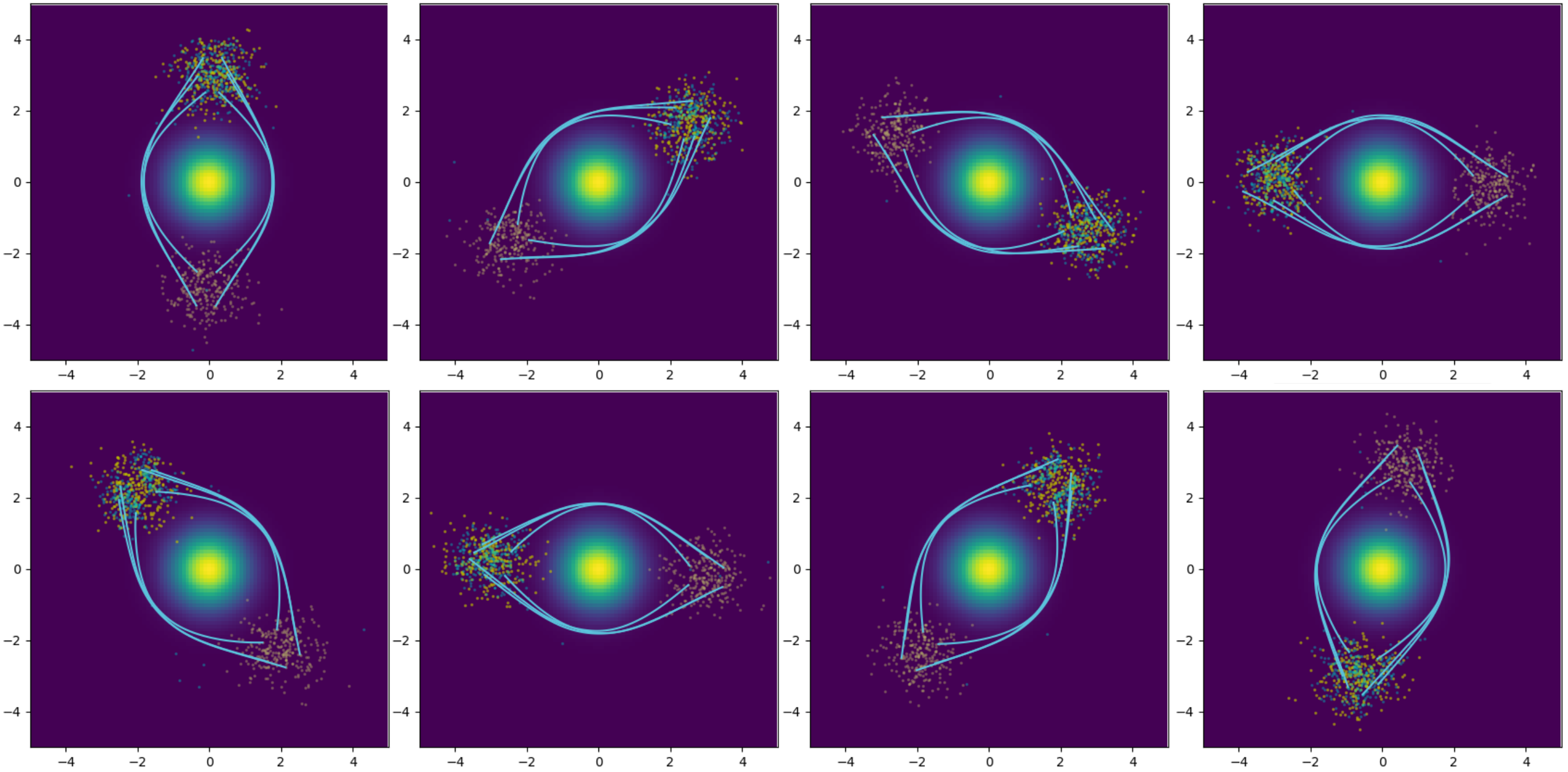}
\caption{Learned trajectories in 5(top) and 10(bottom) dimensions for crowd motion projected to the first two components. Each row shows 4 MFG instances. The center density represents the obstacle $Q(x)$. The khaki, yellow, and cyan points are samples from $P_0, P_1$, and $G_\theta(P_0,P_1)(\cdot, T)_*P_0$, respectively. For each instance, we plot six landmark trajectories to outline the flow of agents.}
\label{fig:crowd_motion}
\end{figure}

\begin{table}[h]
\vskip -10pt
\begin{center}
\begin{small}
\begin{sc}
\begin{tabular}{lccccccccccc}
\toprule
Method & Dim & MFG Cost & $L$ & $\Is$ & $\Ms$ & Time(s)\\
\midrule
APAC-Net\cite{MFG_GAN} & 2 & 3.616 & 28.106 & 0.098 & 0.707 & 9979 \\
APAC-Net & 5 & 4.030 & 29.171 & 0.099 & 1.013 & 10291 \\
APAC-Net & 10 & 4.199 & 28.541 & 0.099 & 1.246 & 10328 \\
APAC-Net & 20 & 4.722 & 32.187 & 0.102 & 1.401 & 10184 \\
\midrule
MFGNet\cite{NN_MFP} & 2 & 3.246 & 28.942 & 0.086 & 0.266 & 2838 \\
MFGNet & 5 & 3.325 & 29.666 & 0.086 & 0.273 & 2884 \\
MFGNet & 10 & 3.453 & 30.828 & 0.086 & 0.284 & 2870 \\
MFGNet & 20 & 3.712 & 33.201 & 0.086 & 0.306 & 2874 \\
\midrule
MFG-NF\cite{MFG_NF} & 2 & 3.252 & 28.951 & 0.086 & 0.266 & 7178 \\
MFG-NF & 5 & 3.336 & 29.611 & 0.086 & 0.273 & 7240 \\
MFG-NF & 10 & 3.460 & 30.888 & 0.086 & 0.283 & 7358 \\
MFG-NF & 20 & 3.717 & 33.261 & 0.086 & 0.305 & 7378 \\
\midrule
Ours & 2 & \textbf{3.246} & 28.944 & 0.086 & 0.267 & \textbf{0.06} \\
Ours & 5 & \textbf{3.325} & 29.642 & 0.086 & 0.275 & \textbf{0.07} \\
Ours & 10 & \textbf{3.453} & 30.819 & 0.086 & 0.284 & \textbf{0.07} \\
Ours & 20 & \textbf{3.712} & 33.194 & 0.086 & 0.307 & \textbf{0.07} \\
\bottomrule
\end{tabular}
\end{sc}
\end{small}
\end{center}
\caption{Baseline results for crowd motion in different dimensions. $L$: Transport cost; $\Is$: Interaction cost; $\Ms$: Terminal cost. Time: Wall clock time in seconds required for solving a single MFG (at inference time for our method).}
\label{tab:crowd_motion}
\vskip -0.1in
\end{table}

In this section, we illustrate the learning of a MFG with nontrivial time-dependent dynamics. Adapting the setup in~\cite{NN_MFP}, we consider a motion planning problem where $P_0, P_1$ are identical Gaussians. Different from the Gaussian setup in Section~\ref{sect:gaussian}, however, an obstacle is placed between the initial and terminal distributions to introduce an interaction term that penalizes agent's proximity to the obstacle. 

Formally, we consider initial and terminal distributions $P_0 = \mathrm{N}(3R_\theta e_2, 0.3 I), P_1 = \mathrm{N}(-3R_\theta e_2, 0.3 I)$, where $R_\theta$ is the rotation matrix for the angle $\theta \sim U[0,2\pi]$. Visually, $P_0, P_1$ are Gaussians with means uniformly sampled from a circle of radius 3 located diametrically from each other. The obstacle $Q(\vx)$ is fixed for different configurations:
\begin{equation}\label{Q}
\begin{split}
    Q(\vx) &= p_N(\vx; 0, \text{diag}(0.5,0.5)),
\end{split}
\end{equation}
where $p_N(\vx; m, \Sigma)$ denotes the probability density function of $\mathrm{N}(m, \Sigma)$. To encourage the agents to avoid the obstacle, we use the following interaction cost:
\begin{equation}\label{crowd_motion_interaction_cost}
\begin{split}
    \mathcal{I}(\Gs(P_0, P_1)(\cdot, t)_*P_0) &\coloneqq \int_{\R^d} Q(\vx) \der (\Gs(P_0, P_1)(\cdot, t)_*P_0)(\vx) = \int_{\R^d} Q(\Gs(P_0, P_1)(\vx,t)) p_0(\vx)\der \vx\\
\end{split}
\end{equation}
If $d>2$, we evaluate $Q$ on the first two components of the agent trajectory.


One way of computing the transport cost is to use an automatic differentiation engine such as PyTorch~\cite{pytorch} to compute $\partial_t G_\theta (X_0, X_1)(\vx,t)$, then use Monte Carlo sampling to approximate the integrals in both time and space. In practice, however, we found that optimizing the analytic derivative to be extremely slow and requiring a prohibitive amount of GPU memory. In addition, the model may learn mappings with sharp transitions in short time spans to create the mirage of having a low transport cost. To mitigate this behavior, one needs to batch a large number of time samples for each gradient step, further exacerbating memory consumption. In light of this observation, we opt to approximate transport and interaction costs with time discretizations, similar to~\cite{MFG_NF}. Note that while the estimated MFG costs are discretized in time, our architecture remains continuous in both space and time.

We present several inferred solution trajectories for different configurations in 5 and 10 dimensions in Figure~\ref{fig:crowd_motion}. Overall, it is evident that the agents have acquired the ability to navigate around the central obstacle and reach their intended destinations while minimizing travel distances. Three additional indicators affirm the accuracy of the results. First, there is symmetry in each problem instance. As the initial and terminal distributions align with the obstacle in their means, the agents split into two groups while traveling whose trajectories mirror each other with respect to the middle obstacle.

Second, we see symmetry across instances - learned agent trajectories for two problems are roughly identical up to a rotation, a phenomenon consistent with the fact that $P_0, P_1$ are identical after the same transformation. Lastly, the right plot of Figure~\ref{fig:err_and_loss} suggests problems in different dimensions share the same optimal MFG cost. This resonates with our expectation as the setup allows optimal trajectories to be constant on all but the first two dimensions. Similar to before, models in all dimensions are trained with $n=256$ samples, marking another MFG instance that can be estimated effectively with a sampling based representation.

Lastly, we compare our approach to existing single-instance neural MFG solvers to showcase its accuracy and efficiency. For fair comparison, we use the following simple terminal cost introduced in~\cite{MFG_GAN} for all methods in lieu of MMD. 
\begin{align}
    \Ms(P) = \int_{\R^d} \|x - x_T\|^2_2 dP(x), x_T = -3R_\theta e_2,
\end{align}
where $R_\theta$ is the rotation matrix for angle $\theta$, and the cost weights are $\lambda_L = 10^{-1}, \lambda_\Is = 1, \lambda_\Ms = 1$. Note that different problem instances share the same optimal cost, so we can obtain the single-instance results by solving the problem with the identity rotation. We adapted the implementation of APAC-Net~\cite{MFG_GAN} and MFG-NF~\cite{MFG_NF} from the their official github repositories. We reimplemented MFGNet~\cite{NN_MFP} in PyTorch since the official codebase is in Julia. We gave our best effort to tune the hyperparameters for each baseline. All wall clock results are benchmarked on a machine with i9-12900K, RTX 4080, and 32GB RAM.


Table~\ref{tab:crowd_motion} documents the computed costs and time to solve each problem for each method. Remarkably, our operator learning approach improves the inference time by multiple orders of magnitude without sacrificing the quality of computed solutions. The impressive time reduction is fundamental to the operator learning framework as it eliminates the need to retrain networks to solve new problems. With its inference time reduced to well under a second, our framework paves the way for exciting \textit{real-time} applications in optimal control and beyond.

\subsection{Path Planning}

\begin{figure}[H]
\centering

\vskip -11pt

\includegraphics[width=1\linewidth]{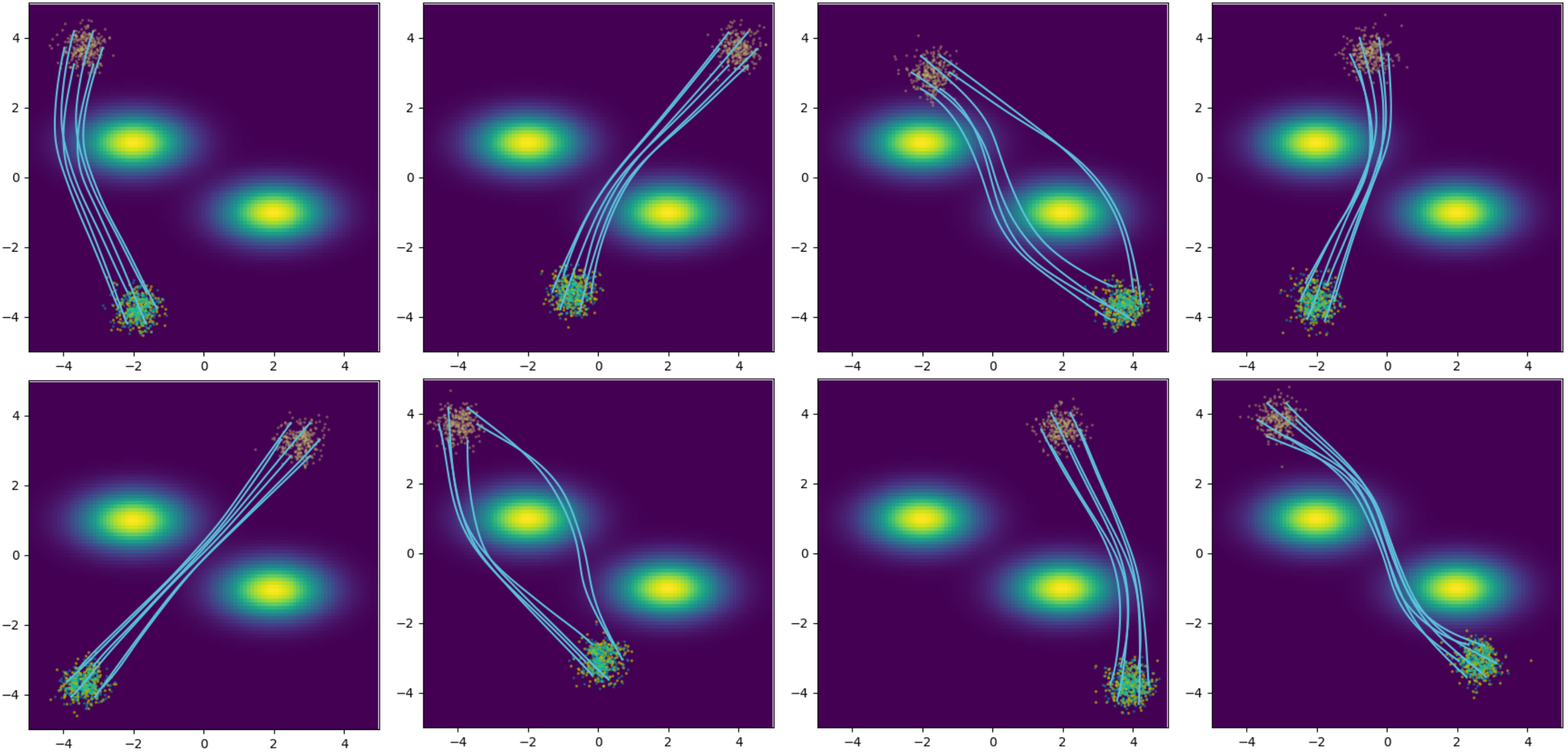}
\caption{Learned trajectories for path planning. Each row shows 4 MFG instances. The two densities in the center represent the obstacle $Q(x)$. The khaki, yellow, and cyan points are samples from $P_0, P_1$, and $G_\theta(P_0,P_1)(\cdot, T)_*P_0$, respectively. For each instance, we plot six landmark trajectories to sketch out the flow of agents.}
\label{fig:path_planning}
\end{figure}

We turn our attention to a setup that resembles a realistic path planning problem. Consider a two dimensional space with two obstacles $Q(\vx)$ placed near the origin. We solve a MFG that seeks to efficiently transport an initial distribution $P_0$ of agents starting above the obstacles to a terminal distribution $P_1$ below. The interaction term is identical to~\eqref{crowd_motion_interaction_cost} and discourages collision with the obstacles. Concretely, we set
\begin{align*}
    Q(\vx) &= \frac{1}{2} p_N(\vx; (-2,1), \text{diag}(1,0.3)) + \frac{1}{2}p_N(\vx; (2,-1), \text{diag}(1,0.3))\\
    P_0 &= \mathrm{N}((m^0_1,m^0_2),0.1I), \text{ where } m^0_1 \sim U[-4,4], m^0_2 \sim U[3,4]\\
    P_1 &= \mathrm{N}((m^1_1,m^1_2),0.1I), \text{ where } m^1_1 \sim U[-4,4], m^1_2 \sim U[-4,-3]
\end{align*}

In Figure~\ref{fig:path_planning}, we provide select optimal trajectories from the learned MFG solution operator on various configurations. Similar to crowd motion, the learned agent flow avoids both obstacles while maintaining a short distance traveled en route to its destination. We emphasize that the learned trajectories do not rely on any expert demonstrations on how to traverse the space.

\section{Proofs}\label{sec:proofs}
In this section, we provide proofs of all statements in Section~\ref{sect:theo_results}.

\subsection{Theorem~\ref{thm:sampling_inv}}


\begin{proof}
$G_\theta (X_0, X_1)(\vx,t)$ is a composition of point-wise MLPs, entry-wise activations, and self-attention blocks. Hence, it is permutation invariant since point-wise MLP, component-wise action and multi-headed attentions are all permutation equivariant, while the final evaluation selects a row in the output by condition and is thus permutation invariant. Furthermore, $G_\theta$ is defined for input samples of any size and can be evaluated at any point $x \in \R^d$ in the output domain. 

It remains to show that $G_\theta(X_0,X_1)(\vx,t)$ converges to a continuous operator as the number of samples $n\to \infty$. First, note that point-wise MLPs and entry-wise activations satisfy this property trivially, so it suffices to show the convergence for self-attention blocks. Assuming the input $X\in \R^{n\times d}$ has iid samples from $P$ as its rows, we define a self-attention block to be $H(X)(\vx)$, where $\vx\in \R^d$ is a point in the output domain. To be consistent with our parametrization~\eqref{fig:architecture}, the attention is applied on the concatenated input $\hat{X} =\begin{pmatrix} \vx^\top \\X \\ \end{pmatrix}$, and the first row in the output matrix (as it corresponds to $\vx$) is used as the final output. That is,


\begin{align}
    H(X)(\vx) &= [\text{softmax}(\frac{(\hat{X}A)(\hat{X}B)^\top }{\sqrt{m}})\hat{X}C]_1\\
    &= \sum_{i=1}^{n+1} \text{softmax} (\frac{(\hat{X}A)(\hat{X}B)^\top }{\sqrt{m}})_{1i} (\hat{X}C)_i\\
    &= \sum_{i=1}^{n+1} \frac{\exp{(\frac{(\hat{X}A)(\hat{X}B)^\top }{\sqrt{m}}})_{1i}}{\sum_{j=1}^{n+1} \exp{(\frac{(\hat{X}A)(\hat{X}B)^\top }{\sqrt{m}})}_{1j}} \hat{X}_i^\top  C
\end{align}
where $A,B,C \in \R^{d \times h}$ are trainable matrices, and $m\in \R$ is a hyperparameter. We write $\hat{X}^\top _i$ as the i-th row of $\hat{X}$.   Noting that $(XY^\top )_{ij} = \langle X_i, Y_j \rangle$, we get

\begin{align}
    H(X)(\vx) = \sum_{i=1}^{n+1} \frac{\exp{(\frac{\langle \hat{X}_1^\top A, \hat{X}_i^\top B\rangle}{\sqrt{m}}})}{\sum_{j=1}^{n+1} \exp{(\frac{\langle \hat{X}_1^\top A, \hat{X}_j^\top B\rangle}{\sqrt{m}})}} \hat{X}_i^\top  C
\end{align}

To simplify our notation, define $q_m(x,y) = \exp(\frac{\langle x, y\rangle}{\sqrt{m}})$. We can split the summations between its first term, which is associated with the input $x$, and the rest. 

\begin{align}
    H(X)(\vx) &= \sum_{i=1}^{n+1} \frac{q_m(\hat{X}_1^\top A, \hat{X}_i^\top B)}{\sum_{j=1}^{n+1} q_m(\hat{X}_1^\top A, \hat{X}_j^\top B)} \hat{X}_i^\top  C\\
    &= \frac{q_m(\hat{X}_1^\top A, \hat{X}_1^\top B)}{\sum_{j=1}^{n+1} q_m(\hat{X}_1^\top A, \hat{X}_j^\top B)} \hat{X}_1^\top  C + \sum_{i=2}^{n+1} \frac{q_m(\hat{X}_1^\top A, \hat{X}_i^\top B)}{q_m(\hat{X}_1^\top A, \hat{X}_1^\top B) + \sum_{j=2}^{n+1} q_m(\hat{X}_1^\top A, \hat{X}_j^\top B)} \hat{X}_i^\top  C \\
    &= \frac{1}{n} \frac{q_m(\hat{X}_1^\top A, \hat{X}_1^\top B)}{\frac{1}{n}\sum_{j=1}^{n+1} q_m(\hat{X}_1^\top A, \hat{X}_j^\top B)} \hat{X}_1^\top  C + \frac{1}{n}\sum_{i=2}^{n+1} \frac{q_m(\hat{X}_1^\top A, \hat{X}_i^\top B)}{\frac{1}{n} q_m(\hat{X}_1^\top A, \hat{X}_1^\top B) + \frac{1}{n}\sum_{j=2}^{n+1} q_m(\hat{X}_1^\top A, \hat{X}_j^\top B)} \hat{X}_i^\top  C \\
    &= \frac{1}{n} \frac{q_m(\vx^\top A, \vx^\top B)}{\frac{1}{n}\sum_{j=1}^{n+1} q_m(\vx^\top A, \hat{X}_j^\top B)} x^\top  C + \frac{1}{n}\sum_{i=2}^{n+1} \frac{q_m(\vx^\top A, \hat{X}_i^\top B)}{\frac{1}{n} q_m(\vx^\top A, \vx^\top B) + \frac{1}{n}\sum_{j=2}^{n+1} q_m(\vx^\top A, \hat{X}_j^\top B)} \hat{X}_i^\top  C 
\end{align}
Taking $n\to \infty$ and invoking the law of large numbers, the above converges to a continuous operator $\mathcal{H}(P)(x)$, where
\begin{align}
    \mathcal{H}(P)(x) \coloneqq \E_{\vy \sim P} \frac{q_m(\vx^\top A, \vy^\top B)}{\E_{\vz\sim P} q_m(\vx^\top A, \vz^\top B)} \vy^\top  C =  \E_{\vy \sim P} \frac{\exp{\frac{\langle \vx^\top  A, \vy^\top B \rangle}{\sqrt{m}}}}{\E_{\vz\sim P}\exp{\frac{\langle \vx^\top  A, \vz^\top B \rangle}{\sqrt{m}}}} \vy^\top  C
\end{align}





The result holds for multi-headed attentions as they are compositions of concatenated self-attention outputs and point-wise MLPs.
\end{proof}

\subsection{Theorem~\ref{thm:opt_obj}}

We first present a lemma used in the theorem's proof. It is a standard result in analysis, but we provide a proof for completeness. Our proof partially follows~\cite{stackexchange_proof}.

\begin{lem}\label{lem:measure_zero}
Let $(X,\mathcal{A}, \mu)$ be a measure space. Let $f: X\to \R_{\geq 0}$ be measurable and $A\in \mathcal{A}$ such that $\mu(A) > 0$. Then 
\begin{align*}
    \int_A f(\vx)d\mu(\vx) = 0 \iff f = 0, \mu\text{-a.e. in } A,
\end{align*}
\end{lem}

\begin{proof}

For the forward direction, let $E_n = \{\vx\in A | f(\vx) > \frac{1}{n}\}$ and $E = \bigcup_{n=1}^\infty E_n$. Then $f(\vx) = 0, \forall \vx \in A\backslash E$. If $\mu(E) > 0$, then by the sub-additivity of measures, $\exists n\in \mathbb{N}$ s.t. $\mu(E_n) > 0$. Hence,
\begin{align*}
    \int_A f(\vx)d\mu(\vx) \geq \frac{1}{n}\mu(E_n) > 0
\end{align*}
For the reverse direction, suppose $f(\vx) > 0, \forall \vx\in C\subseteq A$ s.t. $\mu(C) = 0$, then
\begin{align*}
    \int_A f(\vx)d\mu(x) = \int_C f(\vx)d\mu(x) + \int_{A\backslash C} f(\vx)d\mu(\vx) = 0
\end{align*}
\end{proof}

We now prove the theorem.
\begin{proof}
Let $B = \{(P_0,P_1) | \Gs^*(P_0,P_1) > \argmin_F \Ls(P_0, P_1, F)\} \subseteq \Ps(\R^d) \times \Ps(\R^d)$ and denote $B^c = \Ps(\R^d) \times \Ps(\R^d) \backslash B$. Consider another operator $\hat{\Gs}$ such that $\hat{\Gs}\rvert_{B^c} = \Gs^*\rvert_{B^c}$ and $\hat{\Gs}(P_0,P_1) \in \argmin_F \Ls(P_0,P_1,F), \forall (P_0,P_1) \in B$. 

Now suppose $\mu(B) > 0$, then we have:
\begin{equation}\label{training_obj_proof_pt1}
\begin{split}
    \E_{(P_0, P_1) \sim \mu} \Ls_s (P_0, P_1, \Gs^*(P_0,P_1)) &= \int_{\Ps(\R^d) \times \Ps(\R^d)} \Ls_s(P_0,P_1, \Gs^*(P_0,P_1)) d\mu(P_0,P_1)\\
    &= \int_{B} \Ls_s(P_0,P_1, \Gs^*(P_0,P_1)) d\mu(P_0,P_1) + \int_{B^c} \Ls_s(P_0,P_1, \Gs^*(P_0,P_1)) d\mu(P_0,P_1)\\
    &> \int_{B} \Ls_s(P_0,P_1, \hat{\Gs}(P_0,P_1)) d\mu(P_0,P_1) + \int_{B^c} \Ls_s(P_0,P_1, \hat{\Gs}(P_0,P_1)) d\mu(P_0,P_1)\\
    &= \E_{(P_0, P_1) \sim \mu} \Ls_s(P_0,P_1, \hat{\Gs}(P_0,P_1)),
\end{split}
\end{equation}
where the inequality follows from Lemma~\eqref{lem:measure_zero}. 

The derivation shows that $\Gs^* \notin \argmin_{\Gs} \E_{(P_0, P_1) \sim \mu} \Ls_s (P_0, P_1, \Gs(P_0,P_1))$, a contradiction. Therefore, $\mu(B) = 0$, and $\Gs(P_0,P_1) = \hat{\Gs}(P_0,P_1) \in \argmin_F \Ls(P_0,P_1,F)$ almost everywhere in the sense of $\mu$.

\end{proof}

\subsection{Proposition \ref{prop:gaussian_opt_soln}}

    

We first state and prove the following lemma:

\begin{lem}\label{lem:MFG_W2}
The optimizer and optimal value for the interaction-free MFG
\begin{align}
    \inf_{T} \quad \int_{\R^d}  \| T(\vx) - \vx\|_2^2 &p_0(\vx)\der\vx + \lambda \mathcal{M}(T_*P_0)
\end{align}
are identical to those of the following problem
\begin{align}
\inf_{T} \quad \mathbb{W}^2_2 (P_0, T_*P_0) + \lambda \Ms(T_*P_0)
\end{align}
where $\mathbb{W}_2$ denotes the Wasserstein-2 distance~\cite{OT_book}. 

\begin{proof}
    By the celebrated Brenier's theorem~\cite{Brenier_thm, Villani_OT}, we know that any mapping $T:\R^d \to \R^d$ such that $T = \nabla \phi, \phi $ convex is the OT Monge map between $P_0$ and $T_*P_0$. That is,
\begin{align}
   \mathbb{W}^2_2(P_0, T_*P_0) =  \int_{\R^d}  \| T(\vx) - \vx\|_2^2 &p_0(\vx)\der\vx \text{ s.t. } T = \nabla \phi, \phi \text{ convex} 
\end{align}

Adding $\lambda \Ms(T_*P_0)$ to both sides and take the infimum over $T$, we have:

\begin{align}
    \inf_T \quad \int_{\R^d}  &\| T(\vx) - \vx\|_2^2 p_0(\vx)\der\vx + \lambda \Ms(T_*P_0)= \inf_T \quad \mathbb{W}^2_2(P_0, T_*P_0) + \lambda \Ms(T_*P_0)\\
    &\text{s.t. } T = \nabla \phi, \phi \text{ convex}
\end{align}

Lastly, note the problem on the left is equivalent to the same problem without its constraint, because the unconstrained problem is an interaction-free MFG and thus identical to a related OT problem. Thus, its solution obeys the constraint $T = \nabla \phi, \phi $ convex without explicit enforcement~\cite{OT_book}.  
    
\end{proof}

\end{lem}

Now, we proceed to prove the original proposition~\eqref{prop:gaussian_opt_soln}.

\begin{proof}
    By Lemma~\eqref{lem:MFG_W2}, our MFG problem with the linear MMD kernel is
\begin{align}
    \inf_T \quad \mathbb{W}^2_2(P_0, T_*P_0) + \lambda \|\vm_T - \vm\|^2_2
\end{align}

where $\vm_T \coloneqq \E_{x\sim P_0}[T(x)]$ is the mean of the pushforward distribution. Intuitively, if we consider all distributions whose mean is a given distance away from $m_1$, the one having the least transport cost to $P_0$ should simply be a shifted version of $P_0$. To make this rigorous, we recast the above problem as optimizing over distributions $P$:
\begin{align}
    \inf_P \quad \mathbb{W}^2_2(P_0, P) + \lambda \|\vm_P - \vm\|^2_2
\end{align}
where $\vm_P = \E_{\vx\sim P}[\vx]$. This is justified since the pushforward action can transport between any two absolutely continuous distributions~\cite{Probability_foundations}. In addition, it is known that the Wasserstein-2 distance between two distributions can be decomposed into squared distance between their means and the Wasserstein-2 distance between the centered distributions~\cite{OT_book}. Noting that $P_0$ in our setup is already centered, the above problem is equivalent to:
\begin{align}
    &\inf_{\hat{P}, \vm_P} \quad \|0 - \vm_P\|_2^2 +  \mathbb{W}^2_2(P_0, \hat{P}) + \lambda \|\vm_P - \vm\|^2_2 \\
    = &\inf_{\vm_P} \quad \|\vm_P\|_2^2 + \lambda \|\vm_P - \vm\|^2_2
\end{align}

Taking a gradient to solve for the critical points, we obtain $\vm_P^* = \frac{\lambda}{1+\lambda} \vm$. This means the Monge map is $T^*(\vx) = \vx + \vm_P^* = \vx + \frac{\lambda}{1+\lambda} \vm$. Substitute it back to the MFG, we have the optimal value: $v^* = \frac{\lambda}{1+\lambda}\|\vm\|^2_2$

\end{proof}

\subsection{Proposition \ref{prop:gaussian_Monge_est_err}}\label{app:gaussian_Monge_est_err}

    

\begin{proof}
\begin{equation}
\begin{split}
    \E_{\vx} \|\Bar{T}(\vx) - T^*(\vx)\|^2_2 &= \E_{\vx} \|(\vx+\frac{\lambda}{1+\lambda}\Bar{\vx}) - (x+\frac{\lambda}{1+\lambda}\vm)\|^2_2 = (\frac{\lambda}{1+\lambda})^2\E_{\vx} \|\Bar{\vx}-\vm\|^2_2 \\
    &= (\frac{\lambda}{1+\lambda})^2(\E_{\vx} [\|\Bar{\vx}\|^2 - 2\Bar{\vx}^\top  \vm + \|\vm\|^2_2]) \\
    &= (\frac{\lambda}{1+\lambda})^2(\E \|\Bar{\vx}\|^2_2 - 2 \E[\Bar{\vx}^\top ]\vm + \|\vm\|^2_2)\\
    &= (\frac{\lambda}{1+\lambda})^2(\E \|\Bar{\vx}\|^2_2 - \|\vm\|^2_2)
\end{split}
\end{equation}

We have
\begin{equation}
\begin{split}
    \E \|\Bar{\vx}\|^2_2 &= \E \|\frac{1}{n}\sum_i x_i\|^2_2 = \frac{1}{n^2} \E \|\sum_i x_i\|^2\\
    &= \frac{1}{n^2} \E (\sum_i \|x_i\|^2 + \sum_{i\ne j} x_i^\top x_j)\\
    \text{Note that } x_i &\sim N(m, \sigma^2 I) \implies x_{ij} \sim N(m_j, \sigma^2), \text{ and } \|x_i\|^2_2 = \sum_{j=1}^d x_{ij}^2\\
    &= \frac{1}{n^2} \E (\sum_{i,j} x_{ij}^2 + \sum_{i\ne j} x_i^\top x_j)\\
    &= \frac{1}{n^2} \E [\sum_{i,j} x_{ij}^2] + \frac{1}{n^2} \E [\sum_{i\ne j} x_i^\top x_j]\\
    &= \frac{1}{n^2} \sum_{i,j} \E [x_{ij}^2] + \frac{1}{n^2} \sum_{i\ne j} \E[x_i]^\top \E[x_j]\\
    &= \frac{1}{n^2} n \sum_j \E[x_{ij}^2] + \frac{1}{n^2} \sum_{i\ne j} m^\top m\\
    \text{Use: } &\E[X^2] = Var(X) + \E[X]^2 \implies \E[x_{ij}^2] = \sigma^2 + m_j^2\\
    &= \frac{1}{n^2} n \sum_j \sigma^2 + m_j^2 + \frac{1}{n^2} \sum_{i \ne j } \|\vm\|^2_2\\
    &= \frac{d\sigma^2 + \|\vm\|^2_2}{n} + \frac{n-1}{n} \|\vm\|^2_2\\
    &= \frac{d\sigma^2}{n} + \|\vm\|^2_2
\end{split}
\end{equation}

Therefore,
\begin{equation}
\begin{split}
    \E_{\vx} \|\Bar{T}(\vx) - T^*(\vx)\|^2_2 &= (\frac{\lambda}{1+\lambda})^2(\E \|\Bar{\vx}\|^2_2 - \|\vm\|^2_2)\\
    &= (\frac{\lambda}{1+\lambda})^2(\frac{d\sigma^2}{n} + \|\vm\|^2_2 - \|\vm\|^2_2)\\
    &= (\frac{\lambda}{1+\lambda})^2 \frac{d\sigma^2}{n}
\end{split}
\end{equation}

We also have
\begin{equation}
\begin{split}
    \E_{x\sim P_0} \|T^*(\vx)\|^2_2 &= \E[\|\vx+\frac{\lambda}{1+\lambda}\vm\|^2_2]\\
    &= \E[\|\vx\|^2] + 2\frac{\lambda}{1+\lambda}\E[\vx^\top \vm] + (\frac{\lambda}{1+\lambda})^2\|\vm\|^2_2\\
    &= \E[\sum_j x_j^2] + 2\frac{\lambda}{1+\lambda}\E[\vx]^\top \vm + (\frac{\lambda}{1+\lambda})^2\|\vm\|^2_2\\
    &= \sum_j \E[ x_j^2] + [2\frac{\lambda}{1+\lambda} + (\frac{\lambda}{1+\lambda})^2]\|\vm\|^2_2\\
    &= \sum_j \sigma^2 + m_j^2 + [2\frac{\lambda}{1+\lambda} + (\frac{\lambda}{1+\lambda})^2]\|\vm\|^2_2\\
    &= d\sigma^2 + [1+2\frac{\lambda}{1+\lambda} + (\frac{\lambda}{1+\lambda})^2]\|\vm\|^2_2\\
    &= d\sigma^2 + (1+ \frac{\lambda}{1+\lambda})^2\|\vm\|^2_2
\end{split}
\end{equation}

Taking their ratio completes the proof.
\end{proof}

\section{Conclusion}
\label{sec:conclusion}
In this work, we introduce a pioneering methodology to parameterize MFG solution operators. By training on a distribution of MFG configurations, our model effectively leverages shared characteristics between different MFG instances to output solutions of unseen problems with a single forward pass. Our contribution extends further by formalizing and prove the concept of sampling-invariance for our parametrization, thereby establishing its suitability for operator learning. Notably, our approach's ability to train without access to ground truth labels, coupled with its scalability to high dimensions, endows it with versatility across a broad spectrum of practical scenarios.

\bibliography{references}
\bibliographystyle{plain}

\newpage

\appendix

\section{Architecture}\label{app:architecture}

Let $X_0, X_1 \in \R^{n\times d}$ be row-stacked iid samples from $P_0, P_1$, respectively. For an interaction-free MFG, let $y\in \R^d$ be the evaluation result of the output function at $\vx\in \R^d$ for our parametrization $G_\theta(X_0,X_1)$, i.e., $y\coloneqq G_\theta(X_0, X_1)(\vx)$, we have:
\begin{align*}
X_0 &\gets \begin{pmatrix} \vx^\top\\ X_0\end{pmatrix}\\
X_0' &\gets \mathrm{MLP_0}(X_0) \\
X_1' &\gets \mathrm{MLP_1}(X_1) \\
X' &\gets (X_0', X_1')  \\
X' &\xleftarrow{\times N} \mathrm{MHT}(X') \\
y &\gets [X']_{1,:},
\end{align*}
where $\mathrm{MLP_0}, \mathrm{MLP_1}$ are two distinct MLPs. Unless specified otherwise, all MLPs have one hidden layer with width $2048$ and the GELU\cite{GELU} activation. A dropout layer is also added before the hidden layer. In addition, $X' \in \R^{(2n+1)\times h}$, and $[X']_{1,:}$ means taking the first row of $X'$. For all experiments, we use $N=2$ attention blocks with $m=4$ heads each. 

For MFG with dynamics, let $\hat{y}\in \R^d$ be the evaluation result of the output function at $(\vx,t)\in \R^d \times \R$ for $\hat{G}_\theta(X_0,X_1)$, i.e., $y\coloneqq \hat{G}_\theta(X_0, X_1)(\vx,t)$.
\begin{align*}
[X_0]_{i,:} &\gets ([X_0]_{i,:}, t), i=1,...,n\\
[X_1]_{i,:} &\gets ([X_1]_{i,:}, t), i=1,...,n\\
X_0 &\gets \begin{pmatrix} (\vx^\top,t)\\ X_0\end{pmatrix}\\
X_0' &\gets \mathrm{MLP_0}(X_0) \\
X_1' &\gets \mathrm{MLP_1}(X_1) \\
X' &\gets (X_0', X_1')  \\
X' &\xleftarrow{\times N} \mathrm{MHT}(X') \\
\hat{y} &\gets [X']_{1,:},
\end{align*}

Finally, we put $G_\theta(X_0, X_1)(\vx,t) = \hat{G}_\theta(X_0, X_1)(\vx,t) - \hat{G}_\theta(X_0, X_1)(\vx,0)+ x $ to enforce $G_\theta(X_0, X_1)(\vx,0) = \vx$. Note that the optimal MFG trajectories may have spatially dependent dynamics, so we cannot use decoupled parametrizations such as $G_\theta(X_0, X_1)(\vx,t) = (1-f_{\theta_1}(t))\vx + f_{\theta_1}(t)g_{\theta_2}(\vx)$, where $\theta = (\theta_1, \theta_2)$.

\section{Experiment Details}

For all experiments, we use the parametrizations detailed in~\eqref{app:architecture} trained with the Adam~\cite{ADAM} optimizer and the cosine learning rate scheduler~\cite{cos_LR}. Other salient hyperparameters are documented in Table~\eqref{tab:hyperparam}. 

Overall, the training hyperparameters require minimal tuning between different datasets, which is a sign of robustness. The variations on $B_d$ and $B_s$ are mainly due to memory constraints. For reference, the crowd motion experiment uses about 24GB of VRAM, so training it with e.g., $B_s=1024$ is too expensive for us. However, we do expect the model to learn a more accurate solution operator with more samples.

\begin{table}[h]
\begin{center}
\begin{small}
\begin{sc}
\begin{tabular}{lccccccccccc}
\toprule
& \multicolumn{3}{c|}{MFG params} & \multicolumn{5}{c}{Training params}\\
\midrule
\toprule
Dataset & $\lambda_L$ & $\lambda_\Is$ & $\lambda_\Ms$ & Iterations & LR & $B_d$ & $B_s$ & $p$\\
\midrule
 Gaussian & 5E-3 & - & 1E0 & 5E4 & 3E-5 & 8 & 1024 & 1E-1 \\
 Gaussian Mixture & 1E-3 & - & 1E0 & 5E4 & 3E-5 & 8 & 1024 & 1E-1 \\
 Crowd Motion & 1E-3 & 1E0 & 1E0 & 2E5 & 3E-5 & 4 & 256 & 0 \\
 Path Planning & 1E-3 & 1E0 & 1E0 & 2E5 & 3E-5 & 4 & 256 & 0 \\
 MNIST & 2E-2 & - & 1E0 & 2E5 & 3E-5 & 16 & 1053 & 1E-1 \\
\bottomrule
\end{tabular}
\end{sc}
\end{small}
\end{center}
\caption{Hyperparameters for all numerical experiments. $\lambda_L, \lambda_\Is, \lambda_Ms$: Weights associated with the MFG transport, interaction, and terminal costs, respectively; LR: Learning rate; $B_d$: Number of $(P_0,P_1)$ used per training batch; $B_s$: Number of samples used to represent each $P_0, P_1$. $p$: Dropout probability }
\label{tab:hyperparam}
\end{table}

A few additional aspects are worth mentioning for MFG with dynamics, i.e., crowd motion and path planning. To approximate the integrals in time, we use Simpson's rule on 10 equidistant points. A fourth-order forward scheme is used on the same grid to approximate $\partial_tG_\theta(X_0,X_1)(\vx,t)$. In addition, we note that for an input $(\vx,t)$, the dropout randomness needs to be shared on the same $x$ with different $t$. Otherwise, $G_\theta(X_0, X_1)(\vx,t) = \hat{G}_\theta(X_0, X_1)(\vx,t) - \hat{G}_\theta(X_0, X_1)(\vx,0)+ \vx $ may not satisfy $G_\theta(X_0,X_1)(\vx,0) = \vx$. Having no straightforward way of achieving this, we removed dropout layers for the parametrization with dynamics.

\end{document}